%% file: main.tex
\documentclass{article} 
\usepackage{iclr2024_conference,times}
\usepackage{graphicx} 
\usepackage{multirow}
\usepackage{booktabs}
\usepackage{caption}
\usepackage{subcaption}
\usepackage{multicol}
\usepackage{amsthm}

\input{math_commands.tex}

\usepackage{hyperref}
\usepackage{url}

\makeatother

\usepackage{algorithm}
\usepackage{algpseudocode}

\errorcontextlines\maxdimen

\makeatletter
    \newcommand*{\algrule}[1][\algorithmicindent]{\makebox[#1][l]{\hspace*{.5em}\thealgruleextra\vrule height \thealgruleheight depth \thealgruledepth}}%
\newcommand*{\thealgruleextra}{}
\newcommand*{\thealgruleheight}{.75\baselineskip}
\newcommand*{\thealgruledepth}{.25\baselineskip}

\newcount\ALG@printindent@tempcnta
\def\ALG@printindent{%
    \ifnum \theALG@nested>0
        \ifx\ALG@text\ALG@x@notext
        \else
            \unskip
            \addvspace{-1pt}
            \ALG@printindent@tempcnta=1
            \loop
                \algrule[\csname ALG@ind@\the\ALG@printindent@tempcnta\endcsname]%
                \advance \ALG@printindent@tempcnta 1
            \ifnum \ALG@printindent@tempcnta<\numexpr\theALG@nested+1\relax
            \repeat
        \fi
    \fi
    }%
\usepackage{etoolbox}
\patchcmd{\ALG@doentity}{\noindent\hskip\ALG@tlm}{\ALG@printindent}{}{\errmessage{failed to patch}}
\makeatother

\newbox\statebox
\newcommand{\myState}[1]{%
    \setbox\statebox=\vbox{#1}%
    \edef\thealgruleheight{\dimexpr \the\ht\statebox+1pt\relax}%
    \edef\thealgruledepth{\dimexpr \the\dp\statebox+1pt\relax}%
    \ifdim\thealgruleheight<.75\baselineskip
        \def\thealgruleheight{\dimexpr .75\baselineskip+1pt\relax}%
    \fi
    \ifdim\thealgruledepth<.25\baselineskip
        \def\thealgruledepth{\dimexpr .25\baselineskip+1pt\relax}%
    \fi
    \State #1%
    \def\thealgruleheight{\dimexpr .75\baselineskip+1pt\relax}%
    \def\thealgruledepth{\dimexpr .25\baselineskip+1pt\relax}%
}

\title{Are you \textbf{SURE}? Enhancing Multimodal Pretraining with Missing Modalities through Uncertainty Estimation}
\author{Duy A. Nguyen \\Siebel School of Computing and Data Science\\
University of Illinois, Urbana-Champaign \\
Illinois, USA \\
\texttt{duyan2@illinois.edu} \\
\And
Quan Huu Do \\College of Engineering and Computer Science\\
VinUniversity\\ 
Hanoi, Vietnam \\
\texttt{quan.dh@vinuni.edu.vn}
\And
Khoa D. Doan \\College of Engineering and Computer Science\\
VinUniversity\\
Hanoi, Vietnam \\
\texttt{khoa.dd@vinuni.edu.vn}
\And
Minh N. Do \\Department of Electrical and Computer Engineering\\
University of Illinois, Urbana-Champaign \\
Illinois, USA \\
\texttt{minhdo@illinois.edu} \\
}
\date{April 2024}

\iclrfinalcopy
\begin{document}

\maketitle

\begin{abstract}
Multimodal learning has demonstrated incredible successes by integrating diverse data sources, yet it often relies on the availability of all modalities - an assumption that rarely holds in real-world applications. Pretrained multimodal models, while effective, struggle when confronted with small-scale and incomplete datasets (i.e., missing modalities), limiting their practical applicability. Previous studies on reconstructing missing modalities have overlooked the reconstruction's potential unreliability, which could compromise the quality of the final outputs. We present \textbf{SURE} (Scalable Uncertainty and Reconstruction Estimation), a novel framework that extends the capabilities of pretrained multimodal models by introducing latent space reconstruction and uncertainty estimation for both reconstructed modalities and downstream tasks.  Our method is architecture-agnostic, reconstructs missing modalities, and delivers reliable uncertainty estimates, improving both interpretability and performance. SURE introduces a unique Pearson Correlation-based loss and applies statistical error propagation in deep networks for the first time, allowing precise quantification of uncertainties from missing data and model predictions. Extensive experiments across tasks such as sentiment analysis, genre classification, and action recognition show that SURE consistently achieves state-of-the-art performance, ensuring robust predictions even in the presence of incomplete data.

\end{abstract}

\input{sections_edited/introduction2}
\input{sections_edited/proposed_method}

\input{sections_edited/experiment}
\input{sections_edited/analysis}
\input{sections_edited/conclusion}
\newpage
\bibliography{arxiv/main}
\bibliographystyle{iclr2024_conference}

\newpage

\appendix
\input{sections_edited/appendix}

\end{document}

%% file: math_commands.tex
\usepackage{amsmath,amsfonts,bm}

\newtheorem{theorem}{Theorem}[section]
\newtheorem{proposition}[theorem]{Proposition}

\newenvironment{derivation}{%
  \par\noindent\textbf{\textit{Derivation}.}
}{%
  \hfill$\qed$\par
}

\def\eqref#1{equation~\ref{#1}}

\def\1{\bm{1}}

\DeclareMathAlphabet{\mathsfit}{\encodingdefault}{\sfdefault}{m}{sl}
\SetMathAlphabet{\mathsfit}{bold}{\encodingdefault}{\sfdefault}{bx}{n}



%% file: sections_edited/introduction2.tex
\section{Introduction} \label{sec:intro}
\textbf{Motivation.} Multimodal learning has emerged as a powerful paradigm for processing raw data from diverse sources and formats, frequently outperforming unimodal approaches \cite{multimodal_better}. While these frameworks achieve state-of-the-art performance across numerous tasks \cite{sota1,sota2,mmml}, their success often hinges on idealized conditions during training and evaluation, assuming access to complete modalities and large-scale datasets. In real-world settings, such as autonomous vehicles or medical centers, these ideal conditions are rarely met due to incomplete or noisy data.

To tackle data bottleneck with small-scaled datasets, leveraging pretrained models has proven highly effective in unimodal learning \cite{resnet,bert}. This approach, however, remains underexplored in multimodal contexts. To demonstrate its potential, we evaluated a state-of-the-art multimodal fusion model, MMML \cite{mmml}, on the CMU-MOSI dataset \cite{cmu_mosi}. Two versions of the model were compared: one initialized with pretrained weights from the larger CMU-MOSEI dataset \cite{cmu_mosei} and another trained from scratch (i.e. vanilla). Performance was assessed across varying dataset sizes (Figure \ref{fig:pretrain_vanilla}), revealing the clear advantages of pretrained models in terms of both efficiency and effectiveness on small-scale datasets.

A significant challenge, however, lies in the inability of pretrained multimodal frameworks to handle missing modalities (i.e. some samples come with incomplete sets of multimodal inputs) during training or evaluation. While various reconstruction techniques can impute the missing modalities, these methods alone fail to address the critical issue of prediction reliability under such conditions. Incorporating uncertainty estimation, therefore, becomes essential—not only for quantifying the reliability of predictions but also for making informed decisions in scenarios where safety and trustworthiness are paramount (e.g., healthcare and autonomous driving).

To highlight the importance of uncertainty estimation, we conducted experiments on the CMU-MOSI evaluation dataset under various modality-missing conditions. A simple reconstruction mechanism was used to approximate missing modalities, and the model’s performance was evaluated using both Mean Squared Error (MSE) and Estimated Uncertainty (Figure \ref{fig:unc_err}). The results showed that while missing modalities degrade performance (especially when text is absent), the model generates meaningful uncertainty estimates that align with its prediction errors. Higher average uncertainties consistently correspond to higher errors, demonstrating the utility of uncertainty estimation as a proxy for prediction reliability in multimodal settings.
\begin{figure}[t!]
\begin{minipage}{0.46\textwidth}
    \centering
    \includegraphics[width=\columnwidth]{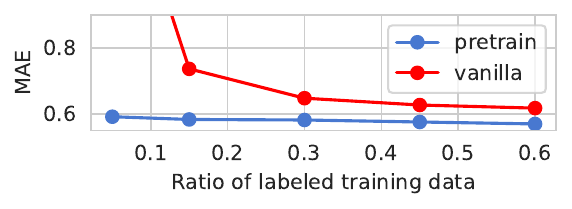}
    \caption{Comparison of pretrained and vanilla MMML framework on CMU-MOSI dataset.}
    \label{fig:pretrain_vanilla}
\end{minipage}
\hfill
\begin{minipage}{0.49\textwidth}
    \centering
    \includegraphics[width=0.8\columnwidth]{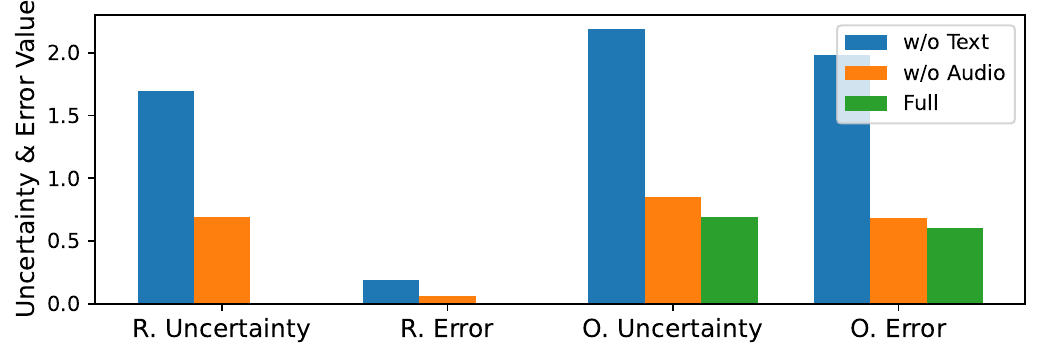}
    \caption{SURE reconstructs missing modalities for final predictions, reporting average errors and uncertainties for both reconstruction (R.) and output (O.).}
    \label{fig:unc_err}
\end{minipage}

\end{figure}

\textbf{Current Literature.} The challenge of missing modalities is a common issue in the training and deployment of multimodal models. Research has primarily addressed this through three approaches: (1) Contrastive loss-based methods, which align latent spaces for cross-modal knowledge transfer \cite{cvpr22,missing_add_1, missing_add_3}; (2) Generative approaches, including VAE-based models \cite{vae_1} and latent space reconstruction \cite{actionmae, missing_add_2}, to approximate missing modalities; and (3) Prompt-based techniques, which utilize trainable prompts to adapt models to various combinations of missing modalities \cite{prompt_1, prompt_3}.
The latter two approaches offer practical solutions for addressing missing modalities and are particularly effective when integrated with pretrained multimodal frameworks. These methods bridge the gap left by missing modalities without requiring extensive modifications to the pretrained models, enabling efficient adaptation to real-world scenarios. However, they overlook the critical need to quantify the uncertainty of reconstructed inputs -- an essential factor for ensuring reliable downstream predictions. 

A growing body of research addresses the second challenge: assessing prediction reliability through uncertainty estimation. Key methods include Bayesian deep learning \cite{wang2019aleatoric,NIPS2017_2650d608}, which models uncertainty directly via output distributions, and post-hoc techniques like Deep Ensembles \cite{deepensemble}, which introduce input perturbations to generate multiple predictions for uncertainty estimation. 
Adapting such approaches for multimodal settings, particularly in conjunction with pretrained frameworks and reconstruction techniques, provides a pathway to more robust and uncertainty-aware models. A more comprehensive review of related works is provided in Appendix \ref{sec:related_works}.

\begin{figure*}[th!]
    \centering
    \includegraphics[width=0.8\textwidth]{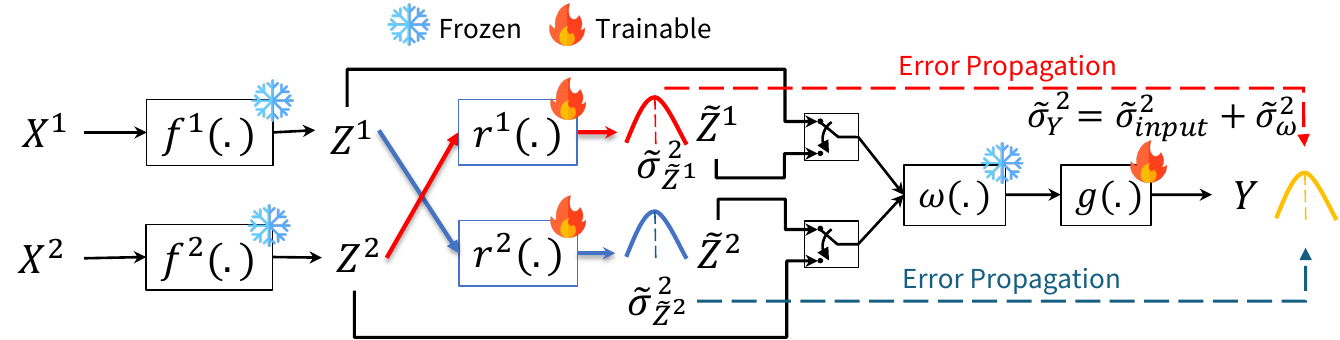}
    \caption{Overview of SURE. Reconstruction modules $r^i(.)$’ are inserted after latent projections $f^i(.)$ and before fusion layers $\omega(.)$ in pretrained multimodal frameworks, reconstructing missing modalities with uncertainties. Reconstructed outputs replace missing inputs while propagating uncertainty, and the final classifier estimates output uncertainty from the model’s inherent variability. The notation $\widetilde{\bullet}$ distinguishes SURE trainable modules' estimates from inputs, outputs, and frozen module's intermediate variables.}
    \label{fig:overview}
\end{figure*}

\textbf{Our approach.} To tackle real-world scenarios with small-scale datasets and missing modalities, we leverage pretrained multimodal frameworks augmented by a simple latent space reconstruction strategy. However, the primary focus of this work is emphasizing the critical need for uncertainty estimation as a central framework enhancement, addressing the challenge of robust multimodal learning in the presence of incomplete data, especially in applications involving \textit{pretrained} multimodal models. We analyze the uncertainties associated with both the reconstruction process of missing modalities and the model's outputs, emphasizing their interdependence and impact on downstream task performance. Specifically, higher reconstruction uncertainty often correlates with higher output uncertainty, which results in greater prediction error (Figure \ref{fig:unc_err}).
To achieve this, three types of uncertainties are investigated:
\begin{itemize}
    \item[(1)] Reconstruction uncertainty for missing modalities.
    \item[(2)] Output uncertainty due to reconstructed inputs.
    \item[(3)] Output uncertainty inherent to the model’s nature.
\end{itemize}
For effective learning of reconstruction uncertainty (1) and aggregated output uncertainty (2 + 3), we introduce a novel loss function based on Pearson Correlation, designed to balance task-specific optimization with uncertainty estimation (Section \ref{sec:method_pearson_loss}). To quantify output uncertainty stemming from reconstructed inputs (2), we adapt Error Propagation \cite{err_prop_1,err_prop_2} for use in deep neural network training, enabling the model to effectively propagate uncertainty from reconstructed inputs to final predictions (Section \ref{sec:method_error_propagation}).
The resulting framework, named SURE (Scalable Uncertainty and Reconstruction Estimation), integrates seamlessly with pretrained multimodal models, providing a scalable and efficient solution for handling missing modalities while generating robust uncertainty estimates. By enabling the model to recognize uncertain predictions and respond with ``I am not sure'' when appropriate, SURE enhances reliability in safety-critical applications. In summary, our key contributions include:
\begin{itemize}
    \item Investigation of the interdependence between reconstruction uncertainty, output reliability, and downstream task performance.
    \item Novel loss function to estimate uncertainty of multimodal models; first adaptation of Error Propagation for Deep Networks.
    \item State-of-the-art results on multiple downstream tasks with accurate uncertainty estimation mechanisms.
\end{itemize}

%% file: sections_edited/proposed_method.tex
\section{Proposed Method} 
\label{sec:method}

We propose SURE, a framework designed to handle small-scale datasets with missing modalities by effectively estimating and leveraging uncertainty.
Let $D = \bigcup_i \{(\textbf{x}_i,\textbf{y}_i)\}$ be the training set with pairs from domain $\mathcal{X}=\mathbb{R}^{n_1}\times...\times \mathbb{R}^{n_M}$ and $\mathcal{Y}=\mathbb{R}^k$. Let denote $\textbf{x}_i=(\textbf{x}_i^1,...,\textbf{x}_i^M)$ is the $i^{th}$ input sample where $\textbf{x}^j_i$ is its $j^{th}$ modality. We address scenarios where certain modalities are missing during training or evaluation, requiring the model to operate reliably despite incomplete inputs. The following sections describe our approach to reconstruct missing modalities and estimate uncertainty in such settings.

\subsection{Overview of SURE Framework}
The proposed pipeline is shown in Figure \ref{fig:overview}. For simplicity, we illustrate the SURE pipeline with two modalities, with its extension to $M$ modalities detailed in Appendix \ref{sec:app_m_modals}. Only the reconstruction modules $r^i(.)$, tailored for each modality, and the final classifier head $g(.)$ require training. This design preserves most of the pretrained framework, except for the final classifier head, consistent with standard unimodal pretraining and fine-tuning procedures.
In this study, we preserve the downstream functionality of pretrained frameworks by reconstructing missing modalities while \textbf{prioritizing the reliability of final predictions}, which is the core objective of SURE.

For the first objective, SURE introduces an efficient reconstruction procedure within the shared latent space. Let $Z^i$ denote the latent representation for the $i^{th}$ modality produced by the pretrained framework's unimodal projectors ($f^i(.)$ in Fig.\ref{fig:overview}). To reconstruct a missing modality, SURE incorporates reconstruction modules $r^i(.)$, each designed specifically for a particular modality. The reconstructor for the $i^{th}$ modality, $r^i(.)$, utilizes the latent representation of another available modality (e.g., $Z^j$) to generate an approximation $\Tilde{Z}^i$ when the $i^{th}$ modality is missing, such that $r^i(Z^j) = \Tilde{Z}^i$ ($i=1,\dots, M; j \ne i$). While these modules are simple and efficient, their primary role is to ensure input completeness for downstream tasks. Additional design details and analysis are discussed in Appendix \ref{sec:app_reconstruct}.

The second and more critical objective is to quantify the reliability of both reconstructed inputs and final predictions. To achieve this, SURE designs the reconstruction modules and classifier head to not only generate outputs but also estimate their corresponding uncertainties. Specifically, the reconstruction module $r^i(.)$ provides an estimate of the $i^{th}$ modality representation $\tilde{Z}^i$ along with its reconstruction uncertainty. Similarly, the classifier head makes final predictions while simultaneously capturing the prediction uncertainty arising from the model’s stochastic nature. Since reconstructed inputs (e.g., $\Tilde{Z}^i$) inherently introduce uncertainty into the output, we extend the analysis of output uncertainty by quantifying how reconstruction uncertainty propagates through the entire model. This is achieved using a tailored adaptation of error propagation techniques (Section \ref{sec:method_error_propagation}). To effectively learn uncertainty estimates, we introduce a novel distribution-free loss function (Section \ref{sec:method_pearson_loss}), which circumvents the limitations of conventional approaches and provides a more flexible measure of uncertainty.

Emphasizing uncertainty estimation as the core innovation, SURE ensures robust and interpretable predictions even in scenarios with incomplete data, while latent reconstruction serves as a straightforward but necessary step to leverage pretrained frameworks for downstream tasks.

\input{sections_edited/method_uncert_est}
\subsection{Error Propagation through Deep Networks}

\label{sec:method_error_propagation}
While the classifier head quantifies output uncertainty from the stochastic nature of pretrained frameworks, a strategy is still required to account for uncertainty introduced by reconstructed inputs. Drawing from the Error Propagation formula \cite{err_prop_1, err_prop_2}, a fundamental principle in scientific measurement and data analysis used to quantify how input uncertainties affect derived quantities, we establish the following key result:
\begin{proposition}
\label{prop:error_propagate}
 Let $\omega(\{Z^i\}_{i \in \mathcal{I}},\{\Tilde{Z}^j\}_{j \in \mathcal{J}})$ be the pretrained model in SURE, where $\mathcal{I}$ and $\mathcal{J}$ represent the sets of indices for which $Z^i$ is available or unavailable, respectively. If each reconstructed modality $\tilde{Z}^i$ has associated uncertainty $\tilde{\sigma}_{\Tilde{Z}_i}^2$, then the total uncertainty propagated from reconstructed inputs is:
  \begin{equation}
      \Tilde{\sigma}_{input}^2 \approx \sum_{i \in \mathcal{J}}\left(\frac{\partial \omega}{\partial \Tilde{Z}_i}\right)^2 \tilde{\sigma}_{\Tilde{Z}_i}^2.
 \end{equation}
 \end{proposition}
\begin{derivation}
The total uncertainty in a function $f(A_1, A_2,\dots, A_n)$ with $n$ variables, each attached with uncertainty $\sigma^2_{A_i}$ ($i=1,\dots,n$), is given by $\sigma^2_f$ and follows:
\begin{equation}
     \sigma_f^2 \approx \sum_{i=1}^n\left(\frac{\partial f}{\partial A_i}\right)^2 \sigma_{A_i}^2,
 \end{equation}
 Applying this to SURE model:
\begin{itemize}
    \itemsep0em
    \item The prediction function $\omega(.)$ depends on both available inputs $Z^i$ and reconstructed inputs $\Tilde{Z}_i$.
    \item Disregard the uncertainty of available modalities, the uncertainty in reconstructed inputs $\Tilde{Z}_i$ propagates to the final output via its partial derivative.
\end{itemize}
Thus, we obtain the result in Proposition \ref{prop:error_propagate}.
\end{derivation}
 Combining $\Tilde{\sigma}_{input}^2$ with the uncertainty stem from $\omega(.)$ stochastic nature, denoted by $\Tilde{\sigma}_{\omega}$ (quantified at SURE's classifier head), we achieve the final output uncertainty:
 \begin{equation}
     \Tilde{\sigma}^2_Y = \Tilde{\sigma}_{input}^2 + \Tilde{\sigma}_{\omega}^2.
 \end{equation}
 This combination follows the Pythagorean theorem for variances, which assumes the aggregated output uncertainty are caused by individual sources \cite{pytago_uncert}: \emph{input-induced uncertainty} (from missing modalities) and \emph{model-intrinsic uncertainty} (from network stochasticity). 
 Lastly, $\Tilde{\sigma}^2_Y$ is learnt to reflect prediction error via $\mathcal{L}_{PCC}$ discussed in Section \ref{sec:method_pearson_loss}.

\textbf{Training process.} In the first phase, reconstruction modules are trained with $\mathcal{L}_{rec}$ using one modality as ground truth and others as input. In the second phase, reconstruction modules are frozen, and the classifier head is trained with $\mathcal{L}_{PCC}$ and $\mathcal{L}_{downstream}$. Detailed training steps are in Algorithm \ref{alg:training} (Appendix \ref{sec:app_m_modals}).

%% file: sections_edited/method_uncert_est.tex
\subsection{Distribution-free Uncertainty Estimation}
\label{sec:method_pearson_loss}

\textbf{Problem Formulation.} Uncertainty estimation aims to quantify the reliability of model predictions, particularly in our case - scenarios with incomplete data (missing modalities). Given an incomplete input $x_i \in \mathcal{X}$, a model $\phi$ with parameters $\theta$ predicts an output $\Tilde{\mathbf{y}}_i$ along with an associated uncertainty estimate $\tilde{\sigma}^2_i$. The goal is to ensure that $\tilde{\sigma}^2_i$ effectively reflects the true prediction error, e.g. $\tilde{\sigma}^2_i \approx ||\Tilde{\mathbf{y}}_i - \mathbf{y}_i||^2$, where $\mathbf{y}_i$ is the corresponding groundtruth.

\textbf{Conventional Approach.} A common approach to uncertainty modeling is optimizing network parameters to estimate $\mathcal{P}_{\textbf{Y}|\textbf{X}}$, often chosen for its closed-form uncertainty estimation. A common choice for $\mathcal{P}_{\textbf{Y}|\textbf{X}}$ is the heteroscedastic Gaussian \cite{upadhyay2022bayescap, kendall2017uncertainties}, where the model learns both a mean prediction $\Tilde{\mathbf{y}}_i$ and variance $\tilde{\sigma}^2_i$ as parameters of a Gaussian distribution:
$
\{\Tilde{\textbf{y}}_i, \Tilde{\sigma}_i\} :=\phi(\textbf{x}_i; \theta).
$
To estimate these parameters, the standard approach maximizes the likelihood:

\begin{equation} 
\theta^* = \underset{\theta}{\operatorname{argmax}} \prod_{i=1}^N \frac{1}{\sqrt{2 \pi \Tilde{\sigma}_i^2}} \exp\left(-\frac{||\Tilde{\textbf{y}}_i - \textbf{y}_i||^2}{2 \Tilde{\sigma}_i^2}\right), 
\end{equation}

which is equivalent to minimizing NLL loss:

\begin{equation} 
\mathcal{L}_{NLL} = \sum_{i=1}^N \frac{||\Tilde{\mathbf{y}}_i - \mathbf{y}_i||^2}{2 \Tilde{\sigma}_i^2} + \frac{\log(\Tilde{\sigma}_i^2)}{2}. 
\label{eq:gaussian_mle} 
\end{equation}

From this formulation, the uncertainty estimate follows:

\begin{equation} 
\Tilde{\sigma}_i^{2*} = ||\Tilde{\mathbf{y}}_i - \mathbf{y}_i||^2 = \Tilde{\epsilon}_i^2. 
\label{eq:gaussian_solution} 
\end{equation}

\textbf{Limitations of NLL Loss.} Despite its widespread use, NLL loss has several key limitations:
\begin{itemize}
    \itemsep0em
    \item \textbf{Assumption of output distribution} – The method assumes $\mathcal{P}_{\textbf{Y}|\textbf{X}}$ follows a particular distribution (e.g. Gaussian), which may not hold in real-world scenarios, leading to unreliable uncertainty estimates.
    \item \textbf{Ill-defined Gradients for Low Errors} – When prediction error $\Tilde{\epsilon}^2 \rightarrow 0$, the gradient of $\mathcal{L}_{NLL}$ is ill-defined, causing trouble to effectively learn uncertainty (detailed analysis in Appendix \ref{sec:app_gaussian_analyses}).
    \item \textbf{Magnitude Constraint} – The model is constrained to scale uncertainty $\Tilde{\epsilon}^2$ directly with error $\Tilde{\epsilon}^2$ , rather than learning a flexible, robust confidence measure. Ideally, uncertainty should reflect confidence independently of error magnitude. 
\end{itemize}

\textbf{Our method.} To circumvent these limitations, we propose a distribution-free loss function that enforces a strong correlation between uncertainty estimates and actual prediction error, without restrictive magnitude constraints. We achieve this by leveraging the Pearson Correlation Coefficient (PCC):

\begin{equation} 
\mathcal{L}_{PCC}(\Tilde{\sigma}^2, \Tilde{\epsilon}^2) = 1 - r(\Tilde{\sigma}^2, \Tilde{\epsilon}^2), \quad \text{where}
\end{equation} 
\begin{equation} 
r(\Tilde{\sigma}^2, \Tilde{\epsilon}^2) = \frac{\sum_{i=1}^N\left(\Tilde{\sigma}^2_i-\mu_{\sigma^2}\right)\left(\Tilde{\epsilon}^2_i-\mu_{\epsilon^2}\right)} {\sqrt{\sum_{i=1}^N\left(\Tilde{\sigma}^2_i-\mu_{\sigma^2}\right)^2} \sqrt{\sum_{i=1}^N\left(\Tilde{\epsilon}^2_i-\mu_{\epsilon^2}\right)^2}}. 
\end{equation}
Here, $\mu_{\sigma^2}$ and $\mu_{\epsilon^2}$ are the means of $\tilde{\sigma^2}$ and $\tilde{\epsilon^2}$, respectively. 
Since Pearson correlation normalizes covariance, the loss is bounded between 0 and 2. A value of 0 indicates perfect alignment between uncertainty $\Tilde{\sigma}^2$ and error $\Tilde{\epsilon}^2$, while higher values imply weaker or inverse relationships.

\begin{theorem}
Let $\bar{\sigma}_i^2$ and $\bar{\epsilon}_i^2$ be the standardized uncertainty and squared error within a mini-batch:
$
\bar{\sigma}_i^2 = \frac{\Tilde{\sigma}_i^2 - \mu_\sigma}{\sqrt{\frac{1}{N-1}\sum_{j=1}^{N}(\Tilde{\sigma}_j^2-\mu_{\sigma^2})^2}},
\bar{\epsilon}_i^2 = \frac{\Tilde{\epsilon}_i^2 - \mu_{\epsilon^2}}{\sqrt{\frac{1}{N-1}\sum_{j=1}^{N}(\Tilde{\epsilon}_j^2-\mu_{\epsilon^2})^2}}
$. Then, $\mathcal{L}_{PCC}$ is approximately equivalent to the MSE between standardized uncertainty and error:
\begin{equation} 
\frac{1}{2N} \sum_{i=1}^N (\bar{\sigma}_i^2 - \bar{\epsilon}_i^2)^2 \approx \mathcal{L}_{PCC}(\Tilde{\sigma}^2, \Tilde{\epsilon}^2). \end{equation}
\end{theorem}
\begin{proof}
    Expanding the LHS:
    $$
    \begin{aligned}
    \frac{1}{2N} \sum_{i=1}^N\left(\bar{\sigma}^2_i-\bar{\epsilon}^2_i\right)^2 
    &=\frac{1}{2N}\left((2N-2)-2 \sum_{i=1}^N \bar{\sigma}^2_i \bar{\epsilon}^2_i\right) \\
    &=\frac{2N-2}{2N}(1-r(\Tilde{\sigma}^2, \Tilde{\epsilon}^2)) \approx \mathcal{L}_{PCC}.
    \end{aligned}
    $$
\end{proof}
\textbf{Key Implications.} Unlike NLL loss, which rigidly forces uncertainty to scale with error, PCC loss enables a more flexible learning of uncertainty while preserving its correlation with error. This relaxed constraint ensures that uncertainty remains a meaningful confidence indicator—predictions with higher uncertainty generally correspond to higher errors, while lower uncertainty is associated with more reliable predictions.
Additionally, we later show in Appendix \ref{sec:app_gaussian_analyses} that our loss function prevents miscalibrated uncertainty in low-error cases by promoting a more stable training process near the optimal solution, which is not the case for ordinary Gaussian NLL loss.

\textbf{Integration with SURE’s Training Objectives.} While $\mathcal{L}_{PCC}$ serves as the core loss for uncertainty estimation, additional loss functions guide reconstruction and downstream task learning. For instance, final loss guiding SURE's reconstruction modules is:
\begin{equation}
    \mathcal{L}_{rec}^i(\Tilde{\mathbf{z}}^i, \Tilde{\sigma}^2_{z^i}) = \frac{1}{N}\sum\|\Tilde{\mathbf{z}}^i - \mathbf{z}^i\|^2 + \lambda \cdot \mathcal{L}_{PCC}(\Tilde{\sigma}^2_{z^i},\Tilde{\epsilon}^2).
\end{equation}
Here, the MSE term guides the learning of reconstructed $\Tilde{\mathbf{z}}^i$, while $\mathcal{L}_{PCC}$ ensures $\sigma^2$ accurately reflects reconstruction error. The weighting factor $\lambda$
balances their contributions. The error $\Tilde{\epsilon}^2$ is defined as MSE for regression tasks (this apply for reconstruction error) or Cross Entropy for classification tasks. A similar loss is used for output uncertainty estimation.  
In parallel, we enhance the estimated output uncertainty by quantifying the uncertainty propagated from the reconstructed input, utilizing Error Propagation through the frozen pretrained network, as detailed in Section \ref{sec:method_error_propagation}.

%% file: sections_edited/experiment.tex
\section{Experiments} \label{sec:experiments}
\subsection{Datasets and Metrics}
We integrate SURE into three pretrained multimodal frameworks and adapt them to smaller-scale datasets with missing modalities during training and evaluation. Detailed integration settings are provided in Appendix \ref{sec:app_settings}.

\textbf{Sentiment Analysis.} For this task, we use MMML \cite{mmml}, a state-of-the-art architecture pretrained on CMU-MOSEI \cite{cmu_mosei}, and fine-tune it on CMU-MOSI \cite{cmu_mosi} with missing modalities. 

\textbf{Book Genre Classification.} SURE is coupled with MMBT \cite{mmbt}, pretrained on MM-IMDB \cite{mm_imdb}, and fine-tune it for book genre classification with text and image data from \cite{book_dts}. 

\textbf{Human Action Recognition.} We use the HAMLET framework \cite{hamlet}, pretrained on MMAct \cite{mmact_dts}, and fine-tune it on UTD-MHAD \cite{utd_mhad_dts} with missing modalities. 

Uncertainty quality is evaluated using Uncertainty Calibration Error (UCE) \cite{uce} and Pearson Correlation Coefficient (PCC) \cite{upadhyay2022bayescap}, which measure alignment between predictive error and uncertainty.

\subsection{Baselines and Experiment Details}
For main comparison, we include reconstruction methods (ActionMAE \cite{actionmae}, DiCMoR \cite{dicmor}, IMDer \cite{imder}) and uncertainty estimation methods (Gaussian Maximum Likelihood \cite{g_mle_1, g_mle_2}, Monte Carlo Dropout \cite{do_1, do_2, do_3}, and Ensemble Learning \cite{deepensemble}). Baselines use the same pretrained models as SURE, and uncertainty methods are tested with a deterministic SURE* variant using MSE loss instead of $\mathcal{L}_{rec}$. Training datasets mask $50\%$ of each modality’s samples, with distinct masks across modalities. Additional baselines and results are discussed in Appendix \ref{sec:app_exps}.

\input{tables/mosi_horizontal}
\input{tables/book_horizontal}
\subsection{Main Results}
Results show pipeline performance with unimodal or full inputs, averaged over three runs. Best and second-best metrics are highlighted in \textbf{\textcolor{red}{red}} and \textcolor{blue}{blue}, respectively.

\textbf{Sentiment Analysis.}
The results for the CMU-MOSI dataset are summarized in Table \ref{tab:mosi}. SURE and its variations consistently outperform recent reconstruction techniques, highlighting their effectiveness in handling missing modalities. SURE’s ability to reconstruct missing data on the fly during training allows every sample to be fully utilized, leading to improved final outputs. Among the modalities, audio appears to be less effective for the downstream task. All methods perform better when text is available compared to when only audio is used, and uncertainty estimation also declines when relying solely on audio.

\textbf{Book Genre Classification.}
Similar to the sentiment analysis task, SURE outperforms recent reconstruction techniques in this classification task (Table \ref{tab:book}), showing a stronger correlation between uncertainty and error for both reconstruction and downstream tasks. In the Book Dataset, the text modality proves to be highly effective for the downstream task, but it contributes less to uncertainty estimation for both reconstruction and downstream tasks.

\textbf{Human Action Recognition.}
As suggested in Table \ref{tab:utd_mhad}, SURE consistently delivers the best performance on downstream tasks across all scenarios. Output uncertainty most closely reflects actual error when the Watch Accel modality is available. However, we observe that a modality effective for downstream task performance may not always contribute equally to uncertainty estimation. This is likely due to the independent nature of error distributions across different modality combinations, which leads to a divergence between downstream task performance and uncertainty estimation. Extended report with every input modalities combination is presented in Appendix \ref{sec:app_missing}.
\input{tables/utd_mhad_horizontal}

\textbf{Summary.}
Compared to other methods, our $\mathcal{L}_{PCC}$ loss avoids strict magnitude constraints, enabling efficient uncertainty learning and accurate error capture from inputs and model stochasticity. Metrics show strong uncertainty-error correlation in reconstruction and downstream tasks, validating SURE's effectiveness.

%% file: tables/mosi_horizontal.tex
\begin{table}[th]
\centering
\caption{Results of different approaches on CMU-MOSI Dataset.}
\label{tab:mosi}
\resizebox{\textwidth}{!}{%
\begin{tabular}{@{}lccccccccccccccccc@{}}
\toprule
\textbf{Model}      & \multicolumn{3}{c}{\textbf{MAE}}                                                                                      & \multicolumn{3}{c}{\textbf{Corr}}                                                                           & \multicolumn{3}{c}{\textbf{F1}}                                                                                       & \multicolumn{3}{c}{\textbf{Acc}}                                                                                    & \multicolumn{2}{c}{\textbf{\begin{tabular}[c]{@{}c@{}}Reconstruct \\ Uncertainty Corr\end{tabular}}} & \multicolumn{3}{c}{\textbf{\begin{tabular}[c]{@{}c@{}}Output \\ Uncertainty Corr\end{tabular}}}                      \\ \midrule
                    & T(ext)                                & A(udio)                               & F(ull)                                & T                                     & A                                    & F                            & T                                     & A                                     & F                                     & T                                     & A                                     & F                                   & T                                                 & A                                                & T                                     & A                                    & F                                     \\ \midrule
\multicolumn{18}{l}{\textit{Modality Reconstruction Techniques:}}                                                                                                                                                                                                                                                                                                                                                                                                                                                                                                                                                                                                                                                                     \\
ActionMAE           & 1.106                                 & 2.146                                 & 1.005                                 & 0.506                                 & 0.155                                & 0.517                        & 0.717                                 & 0.57                                  & 0.719                                 & 0.724                                 & 0.423                                 & 0.725                               & -                                                 & -                                                & -                                     & -                                    & -                                     \\
DiCMoR              & 0.811                                 & 1.227                                 & 1.106                                 & 0.783                                 & 0.427                                & 0.537                        & 0.854                                 & 0.57                                  & 0.65                                  & 0.856                                 & 0.585                                 & 0.654                               & -                                                 & -                                                & -                                     & -                                    & -                                     \\
IMDer               & 0.707                                 & 1.237                                 & 1.106                                 & 0.797                                 & 0.438                                & 0.544                        & 0.846                                 & 0.524                                 & 0.62                                  & 0.846                                 & 0.564                                 & 0.634                               & -                                                 & -                                                & -                                     & -                                    & -                                     \\ \midrule
\multicolumn{18}{l}{\textit{Uncertainty Estimation Techniques:}}                                                                                                                                                                                                                                                                                                                                                                                                                                                                                                                                                                                                                                                                      \\
\begin{tabular}[c]{@{}l@{}}SURE + \\ Gaussian MLE\end{tabular} & {\color[HTML]{FF0000} \textbf{0.589}} & {\color[HTML]{0000FF} 1.133}          & {\color[HTML]{FF0000} \textbf{0.581}} & {\color[HTML]{0000FF} 0.866}          & 0.53                                 & {\color[HTML]{0000FF} 0.871} & 0.88                                  & 0.676                                 & 0.885                                 & 0.879                                 & 0.678                                 & 0.882                               & {\color[HTML]{0000FF} 0.103}                      & 0.013                                            & {\color[HTML]{0000FF} 0.067}          & 0.032                                & 0.059                                 \\
\begin{tabular}[c]{@{}l@{}}SURE + \\ MC DropOut\end{tabular}   & 0.63                                  & 1.153                                 & 0.622                                 & 0.858                                 & 0.556                                & 0.865                        & 0.877                                 & {\color[HTML]{0000FF} 0.686}          & {\color[HTML]{FF0000} \textbf{0.899}} & 0.876                                 & {\color[HTML]{0000FF} 0.684}          & {\color[HTML]{FF0000} \textbf{0.9}} & 0.047                                             & 0.008                                            & 0.013                                 & 0.009                                & {\color[HTML]{0000FF} 0.13}           \\
\begin{tabular}[c]{@{}l@{}}SURE + \\ DeepEnsemble\end{tabular} & {\color[HTML]{0000FF} 0.592}          & {\color[HTML]{FF0000} \textbf{1.071}} & {\color[HTML]{0000FF} 0.582}          & {\color[HTML]{FF0000} \textbf{0.868}} & {\color[HTML]{FF0000} \textbf{0.58}} & {\color[HTML]{0000FF} 0.871} & {\color[HTML]{0000FF} 0.886}          & {\color[HTML]{FF0000} \textbf{0.714}} & 0.889                                 & {\color[HTML]{0000FF} 0.885}          & {\color[HTML]{FF0000} \textbf{0.716}} & 0.888                               & 0.062                                             & {\color[HTML]{0000FF} 0.031}                     & 0.024                                 & {\color[HTML]{0000FF} 0.074}         & 0.082                                 \\ \midrule
\textbf{SURE}                & 0.602                                 & 1.148                                 & 0.583                                 & 0.865                                 & {\color[HTML]{0000FF} 0.557}         & 0.869                        & {\color[HTML]{FF0000} \textbf{0.896}} & 0.685                                 & {\color[HTML]{0000FF} 0.891}          & {\color[HTML]{FF0000} \textbf{0.894}} & {\color[HTML]{0000FF} 0.684}          & {\color[HTML]{0000FF} 0.89}         & {\color[HTML]{FF0000} \textbf{0.739}}             & {\color[HTML]{FF0000} \textbf{0.732}}            & {\color[HTML]{FF0000} \textbf{0.381}} & {\color[HTML]{FF0000} \textbf{0.18}} & {\color[HTML]{FF0000} \textbf{0.485}} \\ \bottomrule
\end{tabular}
}
\end{table}

%% file: tables/book_horizontal.tex
\begin{table}[th]
\centering
\caption{Results of different approaches on Book Dataset.}
\label{tab:book}
\resizebox{0.9\textwidth}{!}{%
\begin{tabular}{@{}lccccccccccc@{}}
\toprule
\textbf{Model}      & \multicolumn{3}{c}{\textbf{F1}}                                                                                       & \multicolumn{3}{c}{\textbf{Acc}}                                                                                      & \multicolumn{2}{c}{\textbf{\begin{tabular}[c]{@{}c@{}}Reconstruct \\ Uncertainty Corr\end{tabular}}} & \multicolumn{3}{c}{\textbf{\begin{tabular}[c]{@{}c@{}}Output \\ Uncertainty Corr\end{tabular}}}                       \\ \midrule
                    & T(ext)                                & I(mage)                               & F(ull)                                & T                                     & I                                     & F                                     & T                                                 & I                                                & T                                     & I                                     & F                                     \\ \midrule
\multicolumn{12}{l}{\textit{Modality Reconstruction Techniques:}}                                                                                                                                                                                                                                                                                                                                                                                                                                            \\
ActionMAE           & 0.277                                 & 0.271                                 & 0.35                                  & 0.186                                 & 0.166                                 & 0.311                                 & -                                                 & -                                                & -                                     & -                                     & -                                     \\
DiCMoR              & 0.202                                 & {\color[HTML]{0000FF} 0.465}          & 0.467                                 & 0.152                                 & {\color[HTML]{0000FF} 0.452}          & 0.454                                 & -                                                 & -                                                & -                                     & -                                     & -                                     \\
IMDer               & 0.204                                 & 0.376                                 & 0.374                                 & 0.155                                 & 0.368                                 & 0.367                                 & -                                                 & -                                                & -                                     & -                                     & -                                     \\ \midrule
\multicolumn{12}{l}{\textit{Uncertainty Estimation Techniques: }}                                                                                                                                                                                                                                                                                                                                                                                                                                           \\
SURE + Gaussian MLE & 0.676                                 & 0.238                                 & {\color[HTML]{0000FF} 0.685}          & 0.665                                 & 0.233                                 & 0.672                                 & 0.137                                             & 0.233                                            & {\color[HTML]{0000FF} 0.358}          & {\color[HTML]{0000FF} 0.349}          & {\color[HTML]{0000FF} 0.468}          \\
SURE + MC DropOut   & 0.653                                 & {\color[HTML]{FF0000} \textbf{0.491}} & 0.669                                 & 0.65                                  & {\color[HTML]{FF0000} \textbf{0.466}} & 0.658                                 & {\color[HTML]{0000FF} 0.243}                      & {\color[HTML]{0000FF} 0.334}                     & 0.174                                 & 0.186                                 & 0.41                                  \\
SURE + DeepEnsemble & {\color[HTML]{0000FF} 0.682}          & 0.327                                 & 0.684                                 & {\color[HTML]{FF0000} \textbf{0.673}} & 0.31                                  & {\color[HTML]{0000FF} 0.673}          & 0.128                                             & 0.135                                            & 0.144                                 & 0.214                                 & 0.227                                 \\ \midrule
\textbf{SURE}       & {\color[HTML]{FF0000} \textbf{0.683}} & 0.413                                 & {\color[HTML]{FF0000} \textbf{0.696}} & {\color[HTML]{0000FF} 0.671}          & 0.401                                 & {\color[HTML]{FF0000} \textbf{0.688}} & {\color[HTML]{FF0000} \textbf{0.637}}             & {\color[HTML]{FF0000} \textbf{0.833}}            & {\color[HTML]{FF0000} \textbf{0.373}} & {\color[HTML]{FF0000} \textbf{0.481}} & {\color[HTML]{FF0000} \textbf{0.474}} \\ \bottomrule
\end{tabular}
}
\end{table}

%% file: tables/utd_mhad_horizontal.tex
\begin{table}[th]
\centering
\caption{Results of different approaches on UTD-MHAD Dataset.}
\label{tab:utd_mhad}
\resizebox{\textwidth}{!}{%
\begin{tabular}{@{}lccccccccccccccc@{}}
\toprule
\textbf{Model}                                                 & \multicolumn{4}{c}{\textbf{F1}}                                                                                                                              & \multicolumn{4}{c}{\textbf{Acc}}                                                                                                                             & \multicolumn{3}{c}{\textbf{\begin{tabular}[c]{@{}c@{}}Reconstruct \\ Uncertainty Corr\end{tabular}}}                  & \multicolumn{4}{c}{\textbf{\begin{tabular}[c]{@{}c@{}}Output \\ Uncertainty Corr\end{tabular}}}                                                              \\ \midrule
                                                               & V(ideo)                              & A(ccel)                               & G(yro)                                & F(ull)                                & V                                     & A                                     & G                                     & F                                    & V                                     & A                                     & G                                     & V                                     & A                                    & G                                     & F                                     \\ \midrule
\multicolumn{16}{l}{\textit{Modality Reconstruction Techniques:}}                                                                                                                                                                                                                                                                                                                                                                                                                                                                                                                                                                                                                   \\
ActionMAE                                                      & 0.044                                & 0.204                                 & 0.303                                 & 0.531                                 & 0.059                                 & 0.231                                 & 0.311                                 & 0.537                                & -                                     & -                                     & -                                     & -                                     & -                                    & -                                     & -                                     \\
DiCMoR                                                         & 0.069                                & {\color[HTML]{FF0000} \textbf{0.473}} & 0.52                                  & 0.653                                 & 0.033                                 & 0.408                                 & 0.472                                 & 0.636                                & -                                     & -                                     & -                                     & -                                     & -                                    & -                                     & -                                     \\
IMDer                                                          & 0.089                                & 0.157                                 & 0.141                                 & 0.687                                 & 0.069                                 & 0.158                                 & 0.145                                 & 0.689                                & -                                     & -                                     & -                                     & -                                     & -                                    & -                                     & -                                     \\ \midrule
\multicolumn{16}{l}{\textit{Uncertainty Estimation Techniques:}}                                                                                                                                                                                                                                                                                                                                                                                                                                                                                                                                                                                                                    \\
\begin{tabular}[c]{@{}l@{}}SURE + \\ Gaussian MLE\end{tabular} & 0.116                                & 0.433                                 & 0.468                                 & 0.693                                 & 0.074                                 & 0.381                                 & 0.387                                 & 0.651                                & 0.166                                 & 0.115                                 & 0.056                                 & 0.122                                 & 0.476                                & 0.147                                 & 0.292                                 \\
\begin{tabular}[c]{@{}l@{}}SURE + \\ MC DropOut\end{tabular}   & 0.156                                & {\color[HTML]{FF0000} \textbf{0.473}} & {\color[HTML]{0000FF} 0.595}          & {\color[HTML]{FF0000} \textbf{0.739}} & 0.09                                  & 0.404                                 & 0.571                                 & 0.718                                & 0.122                                 & 0.135                                 & {\color[HTML]{0000FF} 0.171}          & {\color[HTML]{0000FF} 0.136}          & {\color[HTML]{0000FF} 0.486}         & 0.223                                 & {\color[HTML]{0000FF} 0.512}          \\
\begin{tabular}[c]{@{}l@{}}SURE + \\ DeepEnsemble\end{tabular} & {\color[HTML]{FF0000} \textbf{0.25}} & 0.468                                 & 0.593                                 & 0.737                                 & {\color[HTML]{FF0000} \textbf{0.207}} & {\color[HTML]{FF0000} \textbf{0.453}} & {\color[HTML]{FF0000} \textbf{0.604}} & {\color[HTML]{0000FF} 0.735}         & {\color[HTML]{0000FF} 0.249}          & {\color[HTML]{0000FF} 0.175}          & 0.122                                 & 0.126                                 & 0.421                                & {\color[HTML]{FF0000} \textbf{0.436}} & 0.481                                 \\ \midrule
\textbf{SURE}                                                  & {\color[HTML]{0000FF} 0.161}         & 0.462                                 & {\color[HTML]{FF0000} \textbf{0.607}} & {\color[HTML]{FF0000} \textbf{0.739}} & {\color[HTML]{0000FF} 0.121}          & {\color[HTML]{0000FF} 0.431}          & {\color[HTML]{0000FF} 0.59}           & {\color[HTML]{FF0000} \textbf{0.74}} & {\color[HTML]{FF0000} \textbf{0.878}} & {\color[HTML]{FF0000} \textbf{0.837}} & {\color[HTML]{FF0000} \textbf{0.863}} & {\color[HTML]{FF0000} \textbf{0.226}} & {\color[HTML]{FF0000} \textbf{0.53}} & {\color[HTML]{0000FF} 0.306}          & {\color[HTML]{FF0000} \textbf{0.568}} \\ \bottomrule
\end{tabular}
}
\end{table}

%% file: sections_edited/analysis.tex
\section{Analyses} \label{sec:analyses}
\input{tables/ablation_utd_mhad}
\subsection{Ablation Study.} 
\textbf{Settings.} We analyze the impact of various modules on SURE's performance in both uncertainty estimation and downstream tasks. This analysis includes testing several ablated versions of SURE:
\begin{itemize}
    \item[(1a)] \textbf{Remove $r^i(.)$ modules}: Ignore incomplete samples during training. 
    \item[(1b)] \textbf{Rule-based imputation}: Replace missing modalities with zeros. 
    \item[(2a)] \textbf{Remove uncertainty estimation}: Train $r^i(.)$ with MSE only, no uncertainty estimation. 
    \item[(2b)] \textbf{Remove reconstruction uncertainty}: Train $r^i(.)$ with MSE; omit error propagation logic. 
    \item[(3)] \textbf{Remove pretrained weights}: Reinitialize and train backbone frameworks from scratch. 
\end{itemize}

\textbf{Results.} We present the performance of all SURE variations on the UTD-MHAD dataset in Table \ref{tab:ablation_utd_mhad}. Overall, each ablation negatively impacts SURE’s performance in its respective tasks. Specifically, ignoring missing modalities (1a) or using simple rule-based imputation (1b) significantly reduces downstream task performance, as incomplete yet labeled data remains underutilized and pretrained fameworks is not effectively leveraged.
Additionally, while removing uncertainty estimation logic has a negligible effect on the final task result (2a, 2b), this cause an inability to quantify output uncertainty effectively, ruin our initial effort in producing reliable predictions. Lastly, the results from variation (3) reinforce our motivation: utilizing pretrained weights is far more efficient and beneficial, especially for smaller datasets involving similar tasks.

\subsection{Analyses for estimated uncertainty.}
\textbf{Convergence Analysis.} We visualize the correlation between estimated uncertainties and prediction errors across all training epochs in Figure \ref{fig:convergence}. Compared to the Negative Log-Likelihood Loss (NLL), $\mathcal{L}_{PCC}$ demonstrates superior performance in both convergence speed and final estimation accuracy. Additionally, the shape of the NLL curve suggests instability, as the correlation trend declines after reaching its peak. Although there are fluctuations, our loss maintains an overall upward trend, eventually stabilizing in the final epochs. This experimental results are highly in accordant with our theoretical analysis of convergence points for NLL loss and our proposed loss (Appendix \ref{sec:app_gaussian_analyses}).
\begin{figure}[th]
    \centering
    \includegraphics[width=0.5\columnwidth]{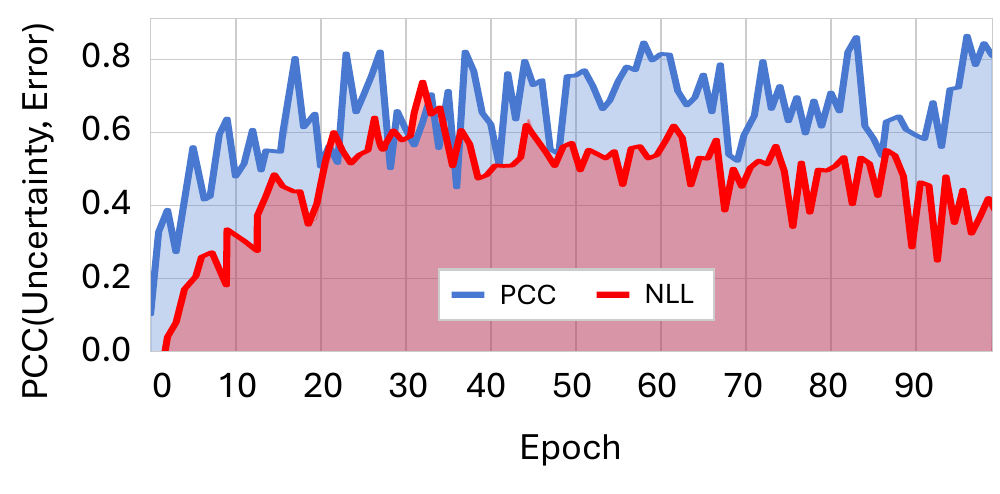}
    \caption{Correlation of estimated uncertainty with prediction error on UTD-MHAD dataset.}
    \label{fig:convergence}
\end{figure}

\begin{figure*}[!t]
    \centering
    \begin{minipage}{0.65\linewidth}
        \centering
        \begin{subfigure}[b]{0.32\textwidth}
        \centering
        \includegraphics[width=\textwidth]{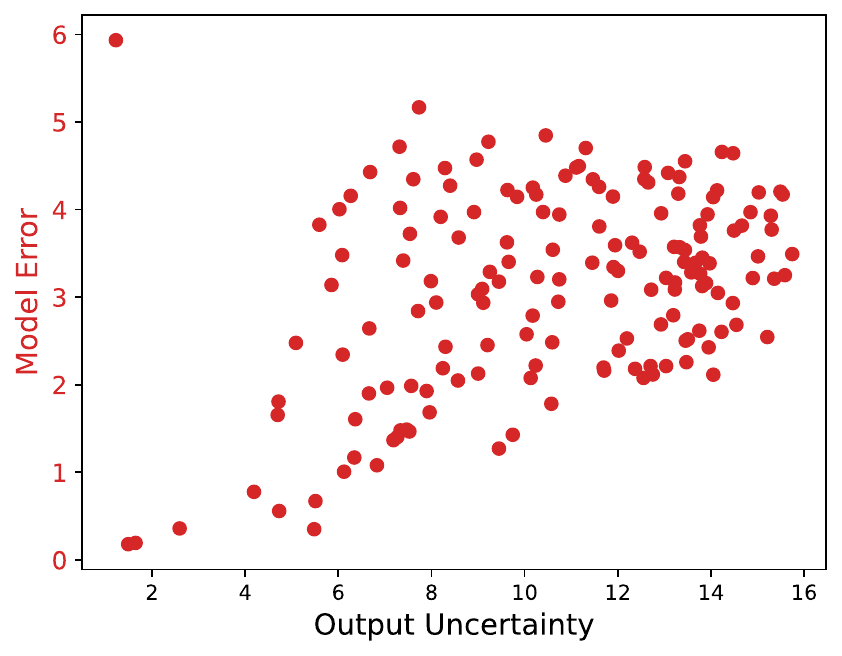}
        \caption{Mod. 0 is missing}
        \end{subfigure}
        \hfill
        \begin{subfigure}[b]{0.32\textwidth}
        \centering
        \includegraphics[width=\textwidth]{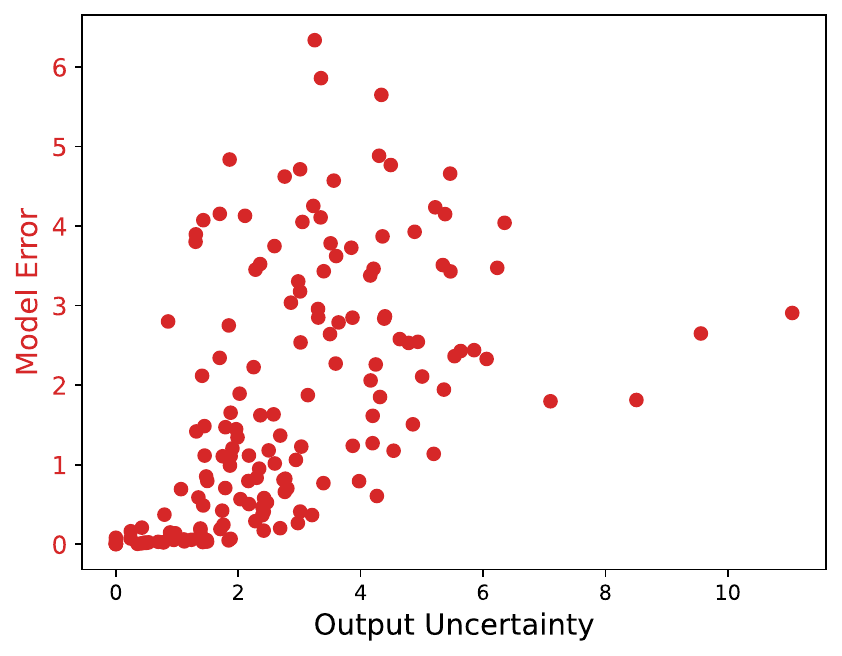}
        \caption{Mod. 1 is missing}
        \end{subfigure}
        \hfill
        \begin{subfigure}[b]{0.32\textwidth}
        \centering
        \includegraphics[width=\textwidth]{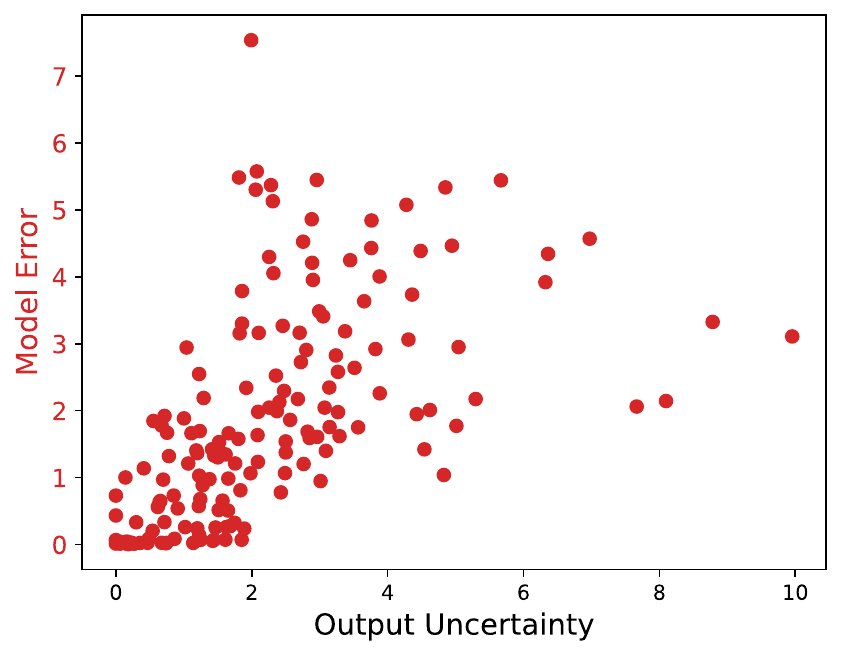}
        \caption{Mod. 2 is missing}
        \end{subfigure}
        \caption{
        Relationship between estimated output uncertainty and output error on UTD-MHAD test dataset.
        }
        \label{fig:missing_analysis_error_out}
    
        \centering
        \begin{subfigure}[b]{0.32\textwidth}
        \centering
        \includegraphics[width=\textwidth]{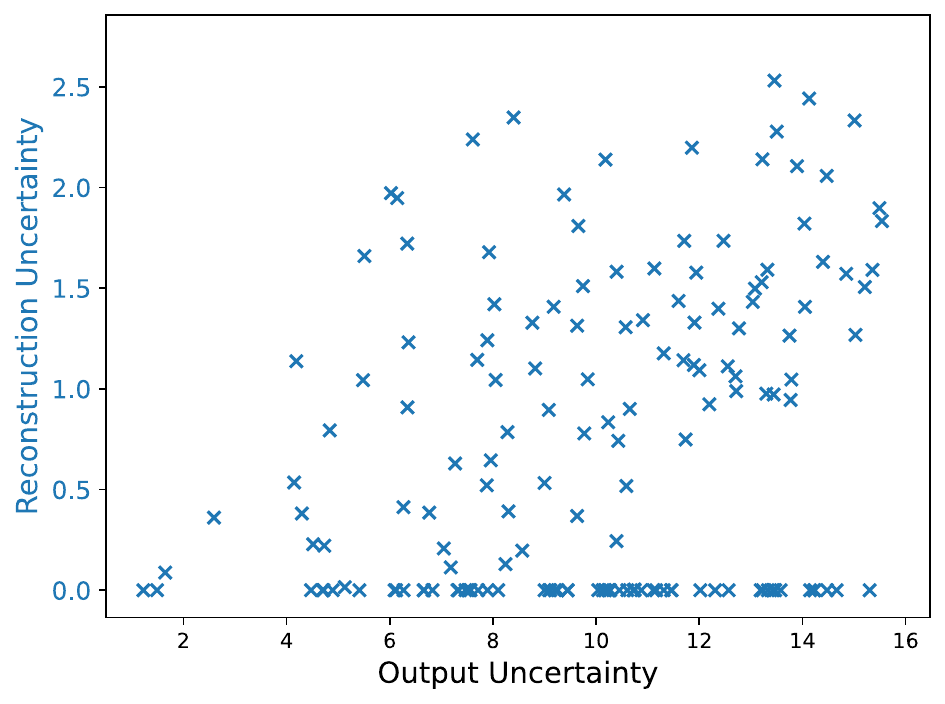}
        \caption{Mod. 0 is missing}
        \end{subfigure}
        \hfill
        \begin{subfigure}[b]{0.32\textwidth}
        \centering
        \includegraphics[width=\textwidth]{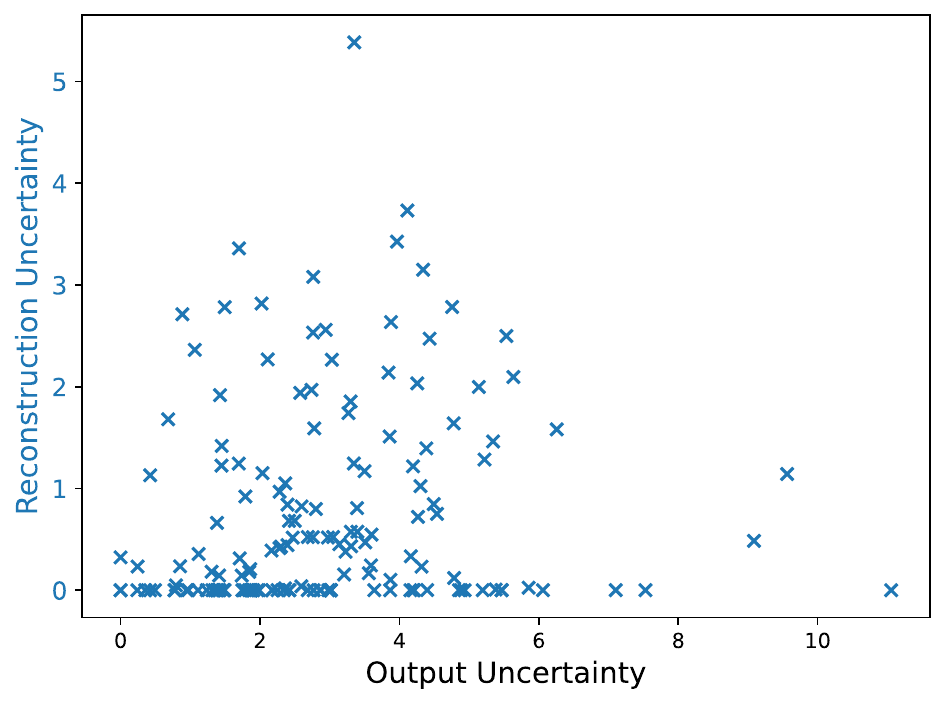}
        \caption{Mod. 1 is missing}
        \end{subfigure}
        \hfill
        \begin{subfigure}[b]{0.32\textwidth}
        \centering
        \includegraphics[width=\textwidth]{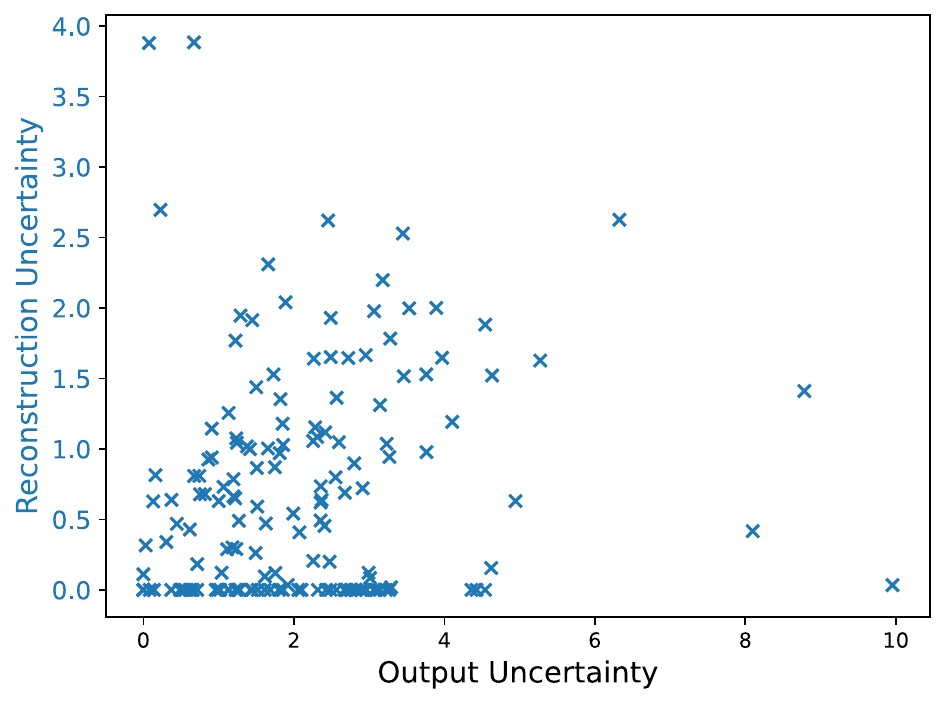}
        \caption{Mod. 2 is missing}
        \end{subfigure}
        \caption{
        Relationship between estimated output uncertainty and reconstruction uncertainty on UTD-MHAD test dataset.
        }
        \label{fig:missing_analysis_in_out}
    \end{minipage}
    \hspace{2em}
    \begin{minipage}{0.25\linewidth}
        \centering
        \begin{subfigure}[b]{\textwidth}
        \centering
        \includegraphics[width=\textwidth]{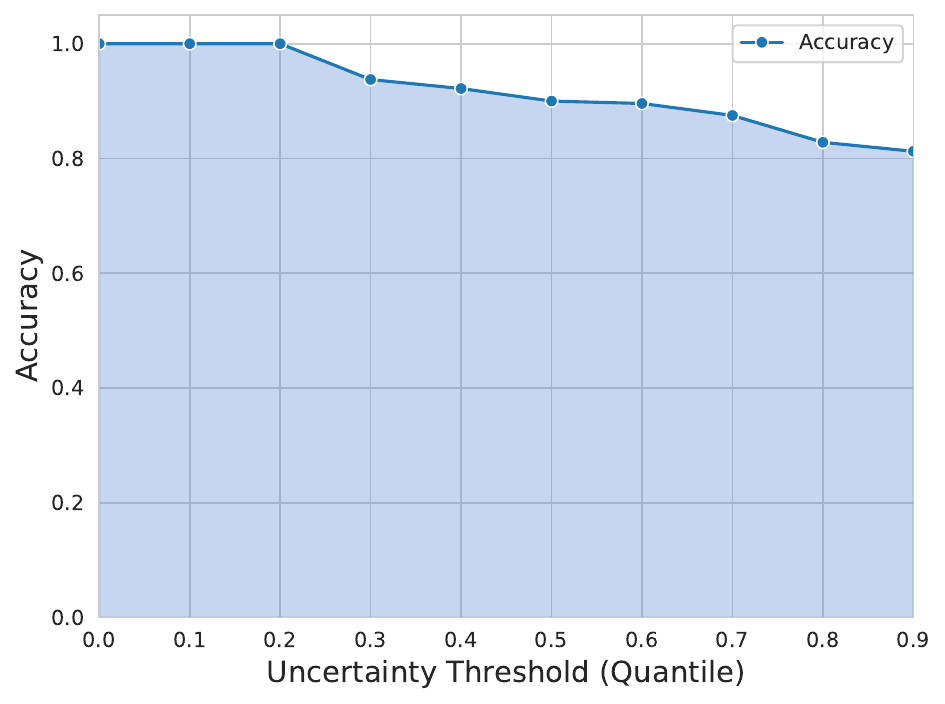}
        \caption{Accuracy-Uncertainty}
        \label{fig:acc_uncert}
        \end{subfigure}
        \hspace{1em}
        \begin{subfigure}[b]{\textwidth}
        \centering
        \includegraphics[width=\textwidth]{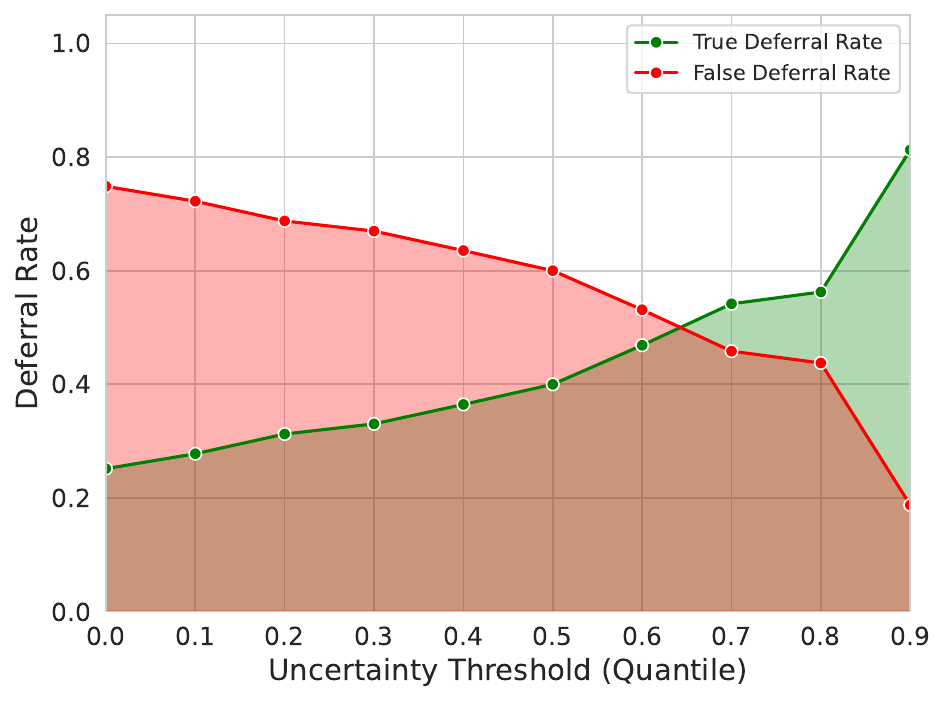}
        \caption{TDR/FDR-Uncertainty}
        \label{fig:rate_uncert}
        \end{subfigure}
        
        \caption{
        Decision Making Process with Uncertainty on UTD-MHAD Dataset.
        }
        \label{fig:decision_making}        
    \end{minipage}
\end{figure*}

\textbf{Reconstruction and Output Uncertainty Analysis.} To better understand the relationship between prediction error, reconstruction uncertainty, and output uncertainty, we visualize these three quantities across all test samples in the UTD-MHAD dataset, with different modalities combinations where each modality is missing. Ideally, the points should cluster along the bottom-left to top-right diagonal, indicating perfect correlation.
With SURE, we observe high efficiency in estimating uncertainty for samples with large prediction errors, which aligns with its intended use as an indicator for potentially error-prone predictions (Figure \ref{fig:missing_analysis_error_out}). Notably, when output uncertainties are high, reconstruction uncertainties tend to be elevated as well (Figure \ref{fig:missing_analysis_in_out}), suggesting that uncertainties arising from the reconstruction process play a significant role in the overall uncertainty estimation.
However, the visualization also indicates a tendency toward overestimating both reconstruction and output uncertainties, highlighting an area for potential improvement in future research.

\subsection{Application: Uncertainty-informed Decision Making with SURE}
\textbf{Settings.} To demonstrate the impact of SURE's uncertainty quantification on decision-making, we simulate this process using a human action recognition task with the UTD-MHAD dataset. SURE is trained with similar settings to those used in the main experiment (Table \ref{tab:utd_mhad}). After training, we use the uncertainty estimates to determine whether the model is confident enough to make a final decision or if it should defer the decision for manual inspection. Different uncertainty thresholds are set based on output uncertainty values from the test dataset. For each threshold, predictions with uncertainty higher than the threshold are deferred, and we record \textbf{Accuracy}, \textbf{True Deferral Rate}, and \textbf{False Deferral Rate} (representing the rate of correctly and incorrectly deferred samples) across all test samples.

\textbf{Results.} Figure \ref{fig:acc_uncert} shows that as more uncertain predictions are deferred, the remaining predictions become more challenging, resulting in a decline in accuracy. This suggests that while the deferral strategy successfully excludes uncertain predictions, it also leaves a set of samples that are inherently harder to predict accurately. 
Additionally, Figure \ref{fig:rate_uncert} demonstrates that as the uncertainty threshold increases, the true deferral rate rises, while the false deferral rate falls. This indicates that the model effectively identifies uncertain predictions (leading to more true deferrals) while reducing unnecessary deferrals. The point at which the true deferral rate surpasses the false deferral rate represents an optimal balance, maximizing decision quality and minimizing unwarranted deferrals. Combining the extended decision-making process under missing modality conditions (as presented in Appendix \ref{sec:app_decision_making}), this analysis indicates that SURE's estimated uncertainty is a reliable indicator for ensuring high prediction quality.

%% file: tables/ablation_utd_mhad.tex
\begin{table}[th]
\centering
\captionof{table}{Results of different SURE's variations of SURE on UTD-MHAD Dataset.}
    \label{tab:ablation_utd_mhad}
\resizebox{0.6\textwidth}{!}{%
\begin{tabular}{@{}lclcccc@{}}
\toprule
                                                  & \textbf{Model}         &             & \multicolumn{1}{l}{\textbf{F1}}       & \multicolumn{1}{l}{\textbf{Acc}}     & \multicolumn{1}{l}{\textbf{\begin{tabular}[c]{@{}l@{}}Reconstruct \\ Uncertainty Corr\end{tabular}}} & \multicolumn{1}{l}{\textbf{\begin{tabular}[c]{@{}l@{}}Output \\ Uncertainty Corr\end{tabular}}} \\ \midrule
                                                  & (1a)                   & Full        & 0.151                                 & 0.098                                & -                                                                                                    & 0.124                                                                                           \\ \cmidrule(l){2-7} 
                                                  &                        & Video       & 0.095                                 & 0.094                                & -                                                                                                    & 0.128                                                                                           \\
                                                  &                        & Accel & 0.059                                 & 0.081                                & -                                                                                                    & 0.322                                                                                           \\
                                                  &                        & Gyro  & 0.408                                 & 0.413                                & -                                                                                                    & 0.122                                                                                           \\
\multirow{-5}{*}{\rotatebox[origin=c]{90}{\begin{tabular}[c]{@{}l@{}}Recontruct \\ Ablation\end{tabular}}}         & \multirow{-4}{*}{(1b)} & Full        & 0.519                                 & 0.525                                & -                                                                                                    & 0.524                                                                                           \\ \midrule
                                                  &                        & Video       & 0.173                                 & 0.117                                & -                                                                                                    & -                                                                                               \\
                                                  &                        & Accel & 0.479                                 & 0.427                                & -                                                                                                    & -                                                                                               \\
                                                  &                        & Gyro  & 0.589                                 & 0.571                                & -                                                                                                    & -                                                                                               \\
                                                  & \multirow{-4}{*}{(2a)} & Full        & \textcolor{blue}{ 0.736}          & \textcolor{blue}{ 0.727}         & -                                                                                                    & -                                                                                               \\ \cmidrule(l){2-7} 
                                                  &                        & Video       & 0.15                                  & 0.113                                & -                                                                                                    & 0.159                                                                                           \\
                                                  &                        & Accel & 0.456                                 & 0.489                                & -                                                                                                    & 0.489                                                                                           \\
                                                  &                        & Gyro  & 0.512                                 & 0.462                                & -                                                                                                    & 0.237                                                                                           \\
\multirow{-8}{*}{\rotatebox[origin=c]{90}{\begin{tabular}[c]{@{}l@{}}Uncertainty \\ Est. Ablation\end{tabular}}} & \multirow{-4}{*}{(2b)} & Full        & 0.637                                 & 0.593                                & -                                                                                                    & 0.511                                                                                           \\ \midrule
                                                  &                        & Video       & 0.031                                 & 0.005                                & 0.684                                                                                                & 0.026                                                                                           \\
                                                  &                        & Accel & 0.226                                 & 0.237                                & 0.675                                                                                                & 0.441                                                                                           \\
                                                  &                        & Gyro  & 0.434                                 & 0.418                                & 0.68                                                                                                 & 0.463                                                                                           \\
\multirow{-4}{*}{\rotatebox[origin=c]{90}{\begin{tabular}[c]{@{}l@{}}Pretraining\\ Ablation\end{tabular}}}            & \multirow{-4}{*}{(3)}  & Full        & 0.615                                 & 0.618                                & -                                                                                                    & 0.472                                                                                           \\ \midrule
                                                  &                        & Video       & 0.161                                 & 0.121                                & \textcolor{red}{ \textbf{0.878}}                                                                & 0.226                                                                                           \\
                                                  &                        & Accel & 0.462                                 & 0.431                                & 0.837                                                                                                & \textcolor{blue}{ 0.53}                                                                     \\
                                                  &                        & Gyro  & 0.607                                 & 0.59                                 & \textcolor{blue}{ 0.863}                                                                         & 0.306                                                                                           \\
                                                  & \multirow{-4}{*}{\textbf{SURE}} & Full        & \textcolor{red}{ \textbf{0.739}} & \textcolor{red}{ \textbf{0.74}} & -                                                                                                    & \textcolor{red}{ \textbf{0.568}}                                                           \\ \bottomrule
\end{tabular}
}
\end{table}

%% file: sections_edited/conclusion.tex
\section{Conclusion}
\label{sec:conclude}
\textbf{Contributions.} This work introduces SURE (Scalable Uncertainty and Reconstruction Estimation), which leverages pretrained multimodal frameworks for small datasets with missing modalities using latent space reconstruction. SURE integrates uncertainty estimation via a Pearson Correlation-based loss and error propagation, ensuring reliable predictions and adaptability across tasks and networks. It achieves state-of-the-art results in both downstream performance and uncertainty estimation.

\textbf{Limitations.} In developing SURE, we observed that certain modalities dominate the reconstruction process, making it easier to predict missing ones but causing significant performance drops when unavailable. This imbalance, unexplored in the current SURE framework, may limit the development of robust reconstruction modules and presents a valuable direction for future work.

%% file: sections_edited/appendix.tex
\section{Appendix}
\input{sections_edited/literature}

\subsection{SURE's Additional Details}

\subsubsection{Negative Log Likelihood Loss for Uncertainty Estimation}
\label{sec:app_gaussian_analyses}

\textbf{Analysis for the convergence of $\mathcal{L}_{NLL}(.,.)$}.

The detailed derivation of the gradient of $\mathcal{L}_{NLL}$ with respect to prediction $\Tilde{y}_i$ is:
\begin{align*}
     \frac{\partial \mathcal{L}_{NLL}(\Tilde{y}, \Tilde{\sigma}^2)}{\partial \Tilde{y}_i} 
     = \frac{\partial }{\partial \Tilde{y}_i}\sum_{i=1}^N \frac{\Tilde{\epsilon_i}^2}{2 \Tilde{\sigma}_i^2}+\frac{\log \left(\Tilde{\sigma}_i^2\right)}{2}
     = \frac{\Tilde{\epsilon_i}}{\Tilde{\sigma}_i^2}
     = \frac{\Tilde{y}_i - y_i}{\Tilde{\sigma}_i^2}.
\end{align*}
Solving $\frac{\partial \mathcal{L}_{NLL}(\Tilde{y}, \Tilde{\sigma}^2)}{\partial \Tilde{y}_i} = 0$ give us the closed form solution $\Tilde{y}^*_i = y_i$ (One can further verify sufficient condition $\frac{\partial^2 \mathcal{L}_{NLL}(\Tilde{y}_i, \Tilde{\sigma}^2_i)}{\partial \Tilde{y}_i^2} = \frac{1}{\Tilde{\sigma}^2_i} > 0$ hold true $\forall i$).

Similarly, gradient of $\mathcal{L}_{NLL}$ with respect to $\Tilde{\sigma}^2_i$ is:
\begin{align*}
    \frac{\partial \mathcal{L}_{NLL}(\Tilde{y}, \Tilde{\sigma}^2)}{\partial \Tilde{\sigma}^2_i}
    = \frac{\partial }{\partial \Tilde{\sigma}^2_i}\sum_{i=1}^N \frac{\Tilde{\epsilon_i}^2}{2 \Tilde{\sigma}_i^2}+\frac{\log \left(\Tilde{\sigma}_i^2\right)}{2}
    = \frac{1}{2\left(\Tilde{\sigma_i}^2\right)^2}\left(\Tilde{\sigma}_i^2-\Tilde{\epsilon}_i^2\right).
\end{align*}
Setting $\frac{\partial \mathcal{L}_{NLL}(\Tilde{y}, \Tilde{\sigma}^2)}{\partial \Tilde{\sigma}^2_i} = 0$ yield $\Tilde{\sigma}_i^{2*}=\Tilde{\epsilon}_i^2$. Verifying the sufficient condition:
\begin{align*}
    \left.\frac{\partial^2 \mathcal{L}_{NLL}}{\partial\left(\Tilde{\sigma}^2_i\right)^2}\right|_{\Tilde{\sigma}^2_i=\Tilde{\epsilon}_i^2}=\frac{1}{2}\left(-\frac{1}{\left(\Tilde{\epsilon}_i^2\right)^2}+\frac{2 \Tilde{\epsilon}_i^2}{\left(\Tilde{\epsilon}_i^2\right)^3}\right)=\frac{1}{2}\left(-\frac{1}{\Tilde{\epsilon}_i^4}+\frac{2}{\Tilde{\epsilon}_i^4}\right)=\frac{1}{2} \cdot \frac{1}{\Tilde{\epsilon}_i^4} > 0
\end{align*}
This test result indicates a local minimum at $\Tilde{\sigma}_i^{2*}=\Tilde{\epsilon}_i^2$.

The issue optimizing $\mathcal{L}_{NLL}(\Tilde{y}, \Tilde{\sigma}^2)$ come up when $\Tilde{\epsilon}_i \rightarrow 0$, this pull the gradient $\frac{\partial \mathcal{L}_{NLL}(\Tilde{y}_i, \Tilde{\sigma}^2_i)}{\partial \Tilde{\sigma}^2_i}$ to the form $\frac{0}{0}$, which is mathematically undefined. This pose a significant issue for gradient-based optimization algorithms like Gradient Descent and cause arbitrary potential issues (gradient vanishing/exploding, numerical under/overflow sensitive to small changes of $\Tilde{\epsilon}^2$, etc).

\textbf{Analysis for the convergence of $\mathcal{L}_{PCC}(.,.)$}.

For this analysis, we focus on the convergence for finding optimal $\Tilde{\sigma}_i^{2*}$ of $\mathcal{L}_{PCC}(.,.)$.
With:
\begin{align*}
    &\mathcal{L}_{PCC}(\Tilde{\sigma}^2,\Tilde{\epsilon}^2) = 1 - r(\Tilde{\sigma}^2,\Tilde{\epsilon}^2); \\
    r(\Tilde{\sigma}^2, \Tilde{\epsilon}^2) &=\frac{\sum_{i=1}^N\left(\Tilde{\sigma}^2_i-\mu_{\sigma^2}\right)\left(\Tilde{\epsilon}^2_i-\mu_{\epsilon^2}\right)}{\sqrt{\sum_{i=1}^N\left(\Tilde{\sigma}^2_i-\mu_{\sigma^2}\right)^2} \sqrt{\sum_{i=1}^N\left(\Tilde{\epsilon}^2_i-\mu_{\epsilon^2}\right)^2}} := \frac{A}{B}.
\end{align*}
We have:
\begin{align*}
    \frac{\partial \mathcal{L}_{PCC}(\Tilde{y}, \Tilde{\sigma}^2)}{\partial \Tilde{\sigma}^2_i} = - \frac{\partial r(\Tilde{y}, \Tilde{\sigma}^2)}{\partial \Tilde{\sigma}^2_i} = - \frac{1}{B}\frac{\partial A}{\partial \Tilde{\sigma}^2_i} + \frac{A}{B^2}\frac{\partial B}{\partial \Tilde{\sigma}^2_i}.
\end{align*}
\begin{align*}
    \frac{\partial A}{\partial \Tilde{\sigma}^2_i} &= \frac{\partial}{\partial \Tilde{\sigma}^2_i} \sum_{j=1}^N\left(\Tilde{\sigma}^2_j-\mu_{\sigma^2}\right)\left(\Tilde{\epsilon}^2_j-\mu_{\epsilon^2}\right) \\
    &= \sum_{j=1}^N (\delta_{ij} - \frac{1}{N})(\Tilde{\epsilon}^2_j-\mu_{\epsilon^2}) \text{ (where } \delta_{ij}=1 \text{ if } i = j \text{ else } 0 \text{)} \\
    &= \Tilde{\epsilon}^2_i - \mu_{\epsilon^2} \text{ (Since } \frac{1}{N} \sum_{j=1}^N \Tilde{\epsilon}^2_i = \mu_{\epsilon^2} \text{)}.
\end{align*}
Also, 
\begin{align*}
    \frac{\partial B}{\partial \Tilde{\sigma}^2_i} &= \frac{\partial}{\partial \Tilde{\sigma}^2_i} \sqrt{\sum_{j=1}^N\left(\Tilde{\sigma}^2_j-\mu_{\sigma^2}\right)^2} \sqrt{\sum_{j=1}^N\left(\Tilde{\epsilon}^2_j-\mu_{\epsilon^2}\right)^2}
\end{align*}
Denoting $\sigma_{\Tilde{\sigma}^2} := \sum_{j=1}^N\left(\Tilde{\sigma}^2_j-\mu_{\sigma^2}\right)^2$, $\sigma_{\Tilde{\epsilon}^2} := \sum_{j=1}^N\left(\Tilde{\epsilon}^2_j-\mu_{\epsilon^2}\right)^2$, we have:
\begin{align*}
    \frac{\partial B}{\partial \Tilde{\sigma}^2_i} &= \sigma_{\Tilde{\epsilon}^2} \frac{1}{2\sigma_{\Tilde{\sigma}^2}} \frac{\partial}{\partial \Tilde{\sigma}^2_i} \sum_{j=1}^N\left(\Tilde{\sigma}^2_j-\mu_{\sigma^2}\right)^2 \\
    &= \sigma_{\Tilde{\epsilon}^2} \frac{1}{2\sigma_{\Tilde{\sigma}^2}} \left[\sum_{j=1}^N 2(\Tilde{\sigma}^2_j-\mu_{\sigma^2})  (\delta_{ij} - \frac{1}{N}) \right] \\
    &= \sigma_{\Tilde{\epsilon}^2} \frac{1}{2\sigma_{\Tilde{\sigma}^2}} \left[2(\Tilde{\sigma}^2_i-\mu_{\sigma^2}) - \frac{2}{N} \sum_{j=1}^N (\Tilde{\sigma}^2_j-\mu_{\sigma^2}) \right] \\
    &= \sigma_{\Tilde{\epsilon}^2} \sigma_{\Tilde{\sigma}^2} \frac{\Tilde{\sigma}^2_i-\mu_{\sigma^2}}{\sigma_{\Tilde{\sigma}^2}^2}
\end{align*}

Assembling the results, we have:
\begin{align*}
    \frac{\partial \mathcal{L}_{PCC}(\Tilde{y}, \Tilde{\sigma}^2)}{\partial \Tilde{\sigma}^2_i} &= - \frac{\partial r(\Tilde{y}, \Tilde{\sigma}^2)}{\partial \Tilde{\sigma}^2_i} = - \frac{1}{B}\frac{\partial A}{\partial \Tilde{\sigma}^2_i} + \frac{A}{B^2}\frac{\partial B}{\partial \Tilde{\sigma}^2_i} \\
    &= - \frac{\Tilde{\epsilon}^2_i - \mu_{\epsilon^2}}{\sigma_{\Tilde{\epsilon}^2} \sigma_{\Tilde{\sigma}^2}} + \frac{\Tilde{\sigma}^2_i-\mu_{\sigma^2}}{\sigma_{\Tilde{\sigma}^2}^2} * \frac{\sum_{j=1}^N\left(\Tilde{\sigma}^2_j-\mu_{\sigma^2}\right)\left(\Tilde{\epsilon}^2_j-\mu_{\epsilon^2}\right)}{\sigma_{\Tilde{\epsilon}^2} \sigma_{\Tilde{\sigma}^2}} \\
    &= - \frac{\Tilde{\epsilon}^2_i - \mu_{\epsilon^2}}{\sigma_{\Tilde{\epsilon}^2} \sigma_{\Tilde{\sigma}^2}} + \frac{\Tilde{\sigma}^2_i-\mu_{\sigma^2}}{\sigma_{\Tilde{\sigma}^2}^2} * r(\Tilde{\sigma}^2,\Tilde{\epsilon}^2) \\
    &= \frac{1}{\sigma_{\Tilde{\sigma}^2}} \left[ 
     \frac{\Tilde{\sigma}^2_i-\mu_{\sigma^2}}{\sigma_{\Tilde{\sigma}^2}} * r(\Tilde{\sigma}^2,\Tilde{\epsilon}^2) - \frac{\Tilde{\epsilon}^2_i - \mu_{\epsilon^2}}{\sigma_{\Tilde{\epsilon}^2}}\right] \\
     &= \frac{1}{\sigma_{\Tilde{\sigma}^2}} \left[ 
      \sigma_{\Tilde{\sigma}^2} * r(\Tilde{\sigma}^2,\Tilde{\epsilon}^2) -\sigma_{\Tilde{\epsilon}^2}\right].
\end{align*}

This last result suggest the gradient $\frac{\partial \mathcal{L}_{PCC}(\Tilde{y}, \Tilde{\sigma}^2)}{\partial \Tilde{\sigma}^2_i}$ involves all standardized variables, which are within a manageable numerical range, reducing the risk of numerical instability. In addition, there is no divisions by $\Tilde{\sigma}^2_i$, hence stabilize the training process even in the event when $\Tilde{\epsilon}^2_i \rightarrow 0$.

\subsubsection{Reconstruction Modules.}
\label{sec:app_reconstruct}
SURE involves a set of reconstruction modules to best leverage the pretrained models' weights. Each reconstruction module is tailored for a specific modality, hence this reconstruction logic is linearly scale with the total number of modalities.

\textbf{Design.} While not mentioned in SURE logic, it should be noted that all $z^j$ are first linearly projected into a shared latent space wherever needed, before passing to the reconstruction modules. This step involves a single matrix multiplication done per modality, and the learnable matrix is trained together with the reconstruction modules. With that, all $r^i(.)$'s are working with the same input latent space, we unify the design of $r^i(.)$'s to be identical across different modalities. Specifically,  the design of reconstruction module $r^i(.)$ is kept as simple as possible, with the major component as Fully Connected layers and ReLU activations as follow:
\begin{equation}
\begin{aligned}
    r^i_{share}(z^j) &= FC(ReLU(FC(z^j))), \\
    r^i_{\mu}(z^j) &= FC(ReLU(FC(ReLU(r^i_{share}(z^j))), \\
    r^i_{\sigma}(z^j) &= SoftPlus(FC(ReLU(FC(ReLU(r^i_{share}(z^j) || r^i_{\mu}(z^j))).
\end{aligned}
\label{eq:reconstruction}
\end{equation}
In Equation \ref{eq:reconstruction}, $||$ denotes the concatenation operation, and $SolfPlus()$ activation is used to ensure the positiveness of returned uncertainty. 

\textbf{Complexity.} Below, we analyze the complexity of the chosen reconstruction modules. Table \ref{tab:hyper_analysis} lists hyper-parameters involved in the analysis.
\begin{table}[ht]
\centering
\caption{$r^i(.)$ related hyper-parameters}
\label{tab:hyper_analysis}
\begin{tabular}{@{}ll@{}}
\toprule
Notation & Description                             \\ \midrule
$M$        & number of modalities                    \\
$L$        & number of FC layers (in total)           \\
$d_i$      & hidden dimension of $i^th$ layer's output \\
$d_0$      & input dimension                         \\ \bottomrule
\end{tabular}
\end{table}

\textbf{Time Complexity.} Assume a single multiplication or summation operation can be performed in unit time ($\mathcal{O}(1)$). We have the calculation for number of operations in a forward pass as follows.
$$
\begin{aligned}
\text{Within the $i^th$ FC layer:} \\
&d_{i-1}*d_i + di, \\
\text{Over $L$ layers:} \\
&\sum_{i=1}^L d_{i-1}*d_i + di. \\
\end{aligned}
$$
In our implementations, we choose the same dimensions for all hidden outputs (same $d=d_i \forall i=1,\dots, L$), and there are $M$ modules $r^i(.)$. With this, the total number of operation is:
$$
M\sum_{i=1}^L d_{i-1}*d_i + di = M*L*d*(d + 1)
=\mathcal{O}(M*L*d^2)
$$
By utilizing matrix product and GPU acceleration, $d^2$ operations can in fact be performed in $\mathcal{O}(1)$ time, make the whole time complexity for individual branches be $\mathcal{O}(M*L)$, which is linearly scaled with $M$.

\textbf{Space Complexity.} Regarding the space complexity, within $i^{th}$ layer, beside the need for storing parameter matrix of size $(d_{i-1} + 1)\times d_i$, output after performing $ReLU$ activation are also stored to later perform back-propagation. Hence, the total number of stored parameters is:
$$(d_{i-1} + 1)*d_i + d_i = (d_{i-1} + 2)*d_i.$$
Following similar derivation with $L$ layers and $M$ branches, replacing $d=d_i \forall i=1,\dots, L$, we have the total space complexity is:
$$M*L*(d + 2)*d = \mathcal{O}(M*L*d^2).$$

Despite utilizing straightforward reconstruction procedure, SURE demonstrates effective reconstruction in the latent space while maintaining an overall additional time and space complexity linearly scaled with $M$ - the number of all modalities and $L$ - the number of FC layers ($6$ in our implementation including both reconstruction and uncertainty heads).

\subsubsection{Extension to M modalities.}
\label{sec:app_m_modals}
For extension to $M$ modalities, we train the reconstruction module using $\mathcal{L}_{rec}$. We use each of the available modalities as the ground-truth output and the rest available modalities as input to predict. 
In the second phase, we freeze all of the reconstruction modules and train the classifier head with $\mathcal{L}_{downstream}$. For each sample with missing modalities, we reconstruct them with remaining available ones, and perform simple average operation to obtain the final reconstruction. 
Algorithm \ref{alg:training} summarize the whole training process of SURE for $m \ge 2$ modalities.
\input{algorithms/training}

\subsubsection{Additional Implementation Details}
\label{sec:app_settings}
\textbf{SURE's Implementation Details.} 
In SURE, we reutilize pretrained multimodal frameworks chosen for specific tasks. The only replacement is the final layers producing prediction, since the classification task might involve different number of classes, and there is an additional output head for estimation of output uncertainty. 

\textbf{Sentiment Analysis.} This task involves predicting the polarity of input data (e.g., video, transcript). We use MMML \cite{mmml} trained on the CMU-MOSEI dataset \cite{cmu_mosei} as the pretrained framework. SURE's reconstruction modules are added right after the projection modules - \emph{Text/Audio feature networks} in original paper's language \cite{mmml}. Their fusion network are kept intact to leverage most pretrained weights as possible. We replace the last fully connected layer - classifier with two layers - one for the final output and one for estimated output uncertainty.

\textbf{Book genre classification.} This task involves classifying book genres based on their titles, summaries (text), and covers (images). We integrate SURE with MMBT \cite{mmbt}, a pretrained framework on the MM-IMDB dataset \cite{mm_imdb}. MMBT is a bitransformer architecture, hence we consider the all the processing before positional embedding and segment embedding as the projection logic (refer to \cite{mmbt} for clearer architecture details), and add our reconstruction modules are inseted after this logic. The remaining transformer logic are considered fusion modules, and kept intact.

\textbf{Human Action Recognition.} This task involves identifying human actions based on recorded videos and sensor data. We use HAMLET framework \cite{hamlet}, pretrained on the large-scale MMAct dataset \cite{mmact_dts} for this task. HAMLET define their projection modules as \emph{Unimodal Feature Encoders} \cite{hamlet}. SURE's reconstruction modules are included right after these encoders, while retain their original MAT module.

\textbf{Baselines' Implementation Details.} 
In our comparative evaluation, we incorporate several state-of-the-art approaches, each representing prominent strategies. The baselines are grouped into two categories, reflecting the key challenges addressed by SURE: (1) Reconstruction methods for missing modalities, and (2) Uncertainty estimation methods. The reconstruction techniques include ActionMAE \cite{actionmae}, DiCMoR \cite{dicmor}, and IMDer \cite{imder}. For uncertainty estimation, we evaluate against the Gaussian Maximum Likelihood method \cite{g_mle_1, g_mle_2}, Monte Carlo Dropout \cite{do_1, do_2, do_3}, and Ensemble Learning \cite{deepensemble}. The original codebases of all baseline implementations are used for best reproducibility.

For all baselines, we also reutilize pretrained multimodal frameworks chosen for specific tasks like the adoptation with SURE. 
In addition, the hidden dimension used within Reconstruction-based baselines are also modified to be the same as those used within SURE for fair comparison.

\subsection{Environment Settings}
\label{sec:app_compute}
All implementations and experiments are performed on a single machine with the following hardware setup: a 6-core Intel Xeon CPU and two NVIDIA A100 GPUs for accelerated training.

Our codebase is primarily built using \emph{PyTorch 2.0}, incorporating \emph{Pytorch-AutoGrad} for deep learning model development and computations. We also use tools from \emph{Scikit-learn}, \emph{Pandas}, and \emph{Matplotlib} to support various experimental functionalities. The original codebase for SURE will be released publicly upon publication.

\subsection{Additional Experiments and Analyses}
\label{sec:app_exps}
\subsubsection{Extended modalities missing scenarios}
\label{sec:app_missing}
\input{tables/utd_mhad_full}
In Table \ref{tab:utd_mhad_full}, we provide a comprehensive evaluation of different frameworks across all combinations of input modalities on the UTD-MHAD dataset. This table expands on the information presented in Table \ref{tab:utd_mhad} in the main text. The reported reconstruction uncertainty for cases with more than one available modality is averaged over all missing modalities (e.g., given (Video + Accel) inputs, the reported reconstruction uncertainty represents the average value for Gyro reconstruction). The results show that SURE consistently delivers the best performance in most uncertainty estimation scenarios while maintaining competitive results for the downstream task, underscoring its robustness across different missing modality situations.

\subsubsection{Extend analysis for inter-relationship between estimated uncertainties and error.}
\label{sec:app_inter_relationship}

\begin{figure}[!t]
    \begin{subfigure}[b]{0.32\textwidth}
    \centering
    \includegraphics[width=\textwidth]{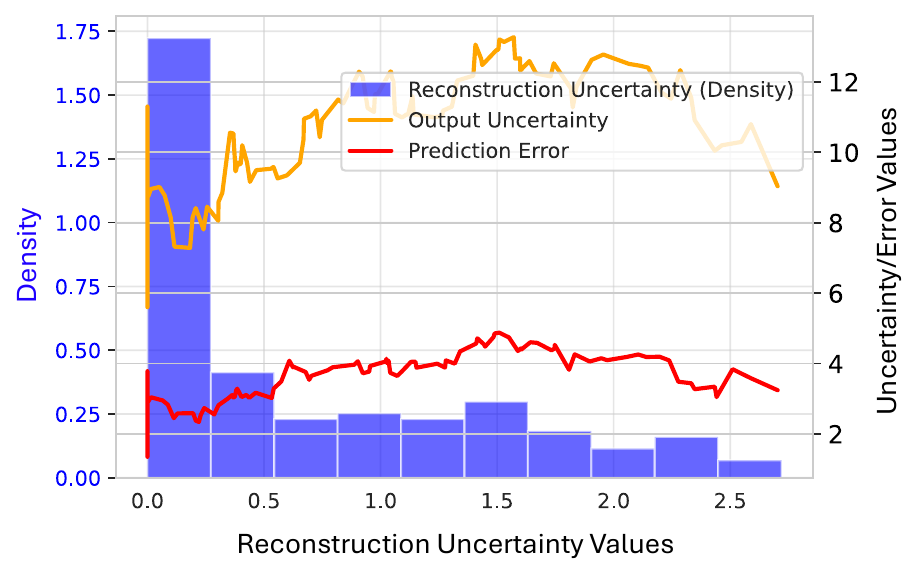}
    \caption{Mod. 0 is missing}
    \end{subfigure}
    \hfill
    \begin{subfigure}[b]{0.32\textwidth}
    \centering
    \includegraphics[width=\textwidth]{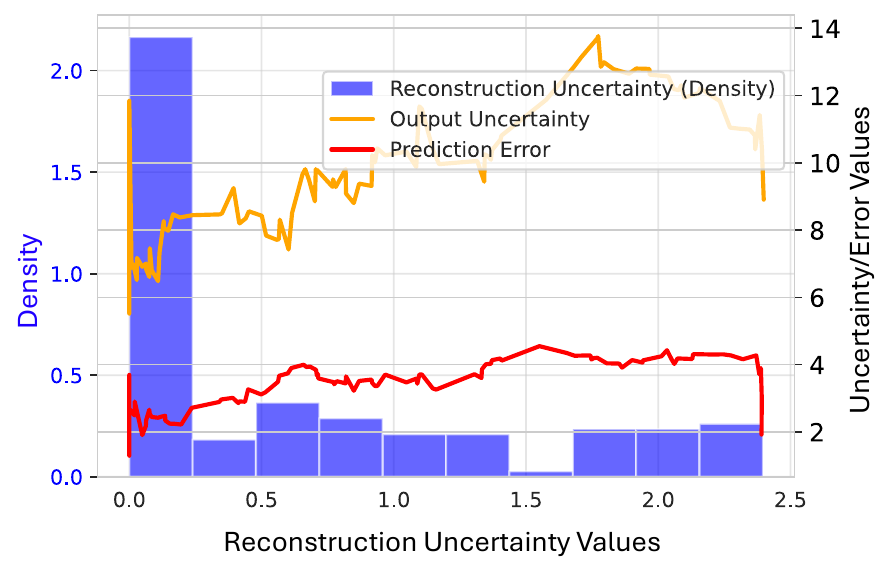}
    \caption{Mod. 1 is missing}
    \end{subfigure}
    \hfill
    \begin{subfigure}[b]{0.32\textwidth}
    \centering
    \includegraphics[width=\textwidth]{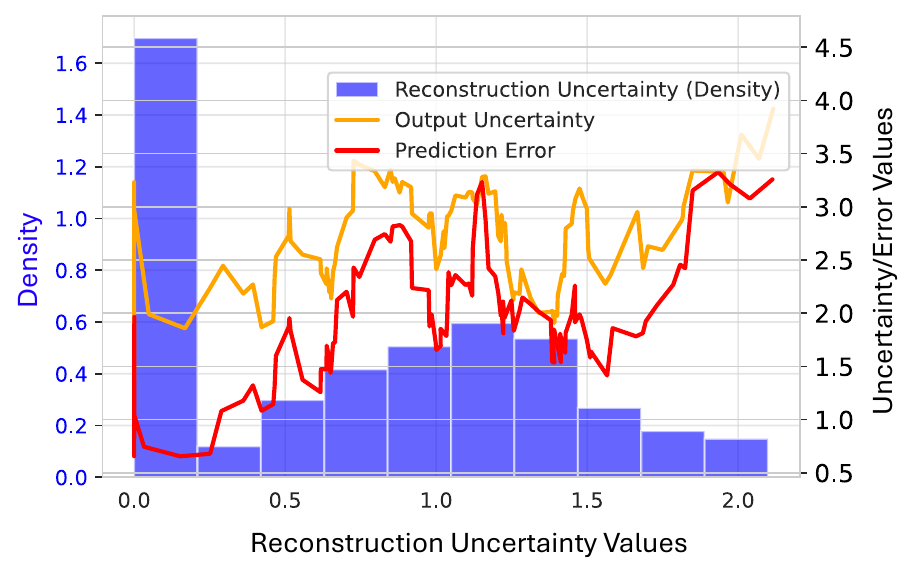}
    \caption{Mod. 2 is missing}
    \end{subfigure}
    \caption{
    Inter-relationship between estimated output uncertainty, reconstruction uncertainty and output error on UTD-MHAD test dataset.
    }
    \label{fig:missing_analysis_error_out_in}
\end{figure}

We further analyze Figure \ref{fig:missing_analysis_error_out_in} to explore the inter-relationship between prediction error, reconstruction uncertainty, and output uncertainty. This visualization includes all test samples from the UTD-MHAD dataset under different modality-missing scenarios, where each modality is systematically excluded. The histogram represents reconstruction uncertainty, while two line plots illustrate output uncertainty and prediction error.

Ideally, the two line plots should display an increasing trend, indicating a positive correlation with reconstruction uncertainty. Using SURE, such a trend is partially observed when modalities 0 and 1 are missing; however, it is less evident when modality 2 is missing. When combined with downstream performance for different input modality combinations, this experiment reveals two key insights:
\begin{itemize}
    \item The Importance of Strong Modalities: Modality 2 (Gyro) plays a critical role in both downstream task performance and reconstructing other modalities. This suggests that stronger modalities are more effective in compensating for or reconstructing missing inputs to solve particular downstream tasks.
    \item Correlations Between Quantities: Output uncertainty is strongly correlated with prediction error, whereas reconstruction uncertainty shows a weaker correlation. This result aligns with our design: output uncertainty is directly trained to reflect prediction error, while reconstruction uncertainty may not be the primary source of error—model limitations and other factors can also contribute significantly.
\end{itemize}
These observations highlight the nuanced dynamics between uncertainties and error quantities produced with SURE. 

\subsubsection{Extend Decision Making Analysis}
\label{sec:app_decision_making}

\begin{figure}[th]
    \centering
    \begin{subfigure}[b]{0.45\textwidth}
        \centering
        \includegraphics[width=\textwidth]{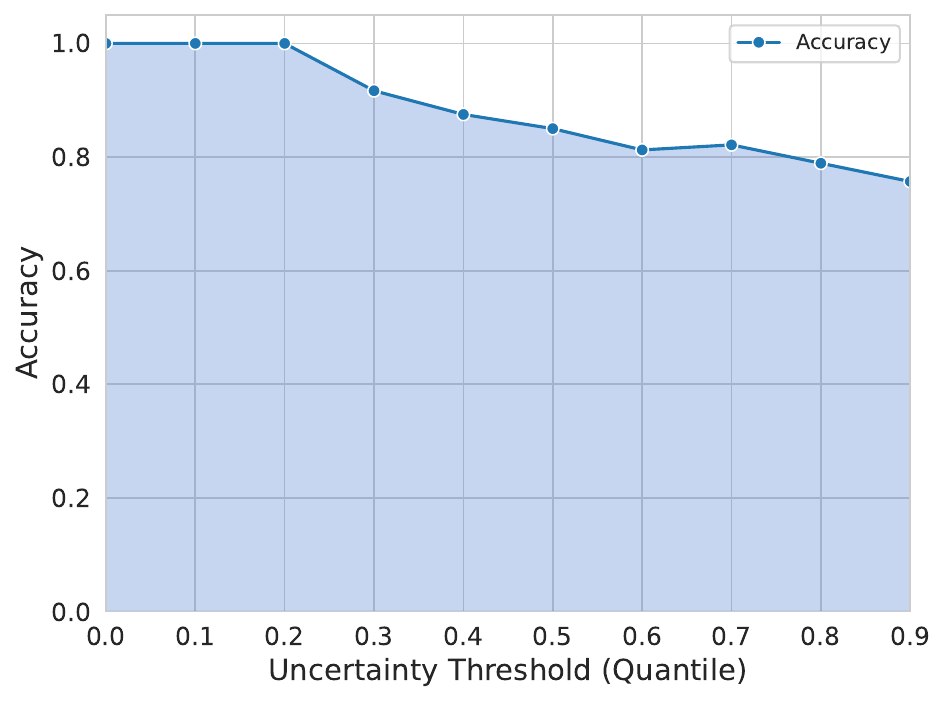}
        \caption{Accuracy-Uncertainty}
        \end{subfigure}
        \hspace{1em}
    \begin{subfigure}[b]{0.45\textwidth}
        \centering
        \includegraphics[width=\textwidth]{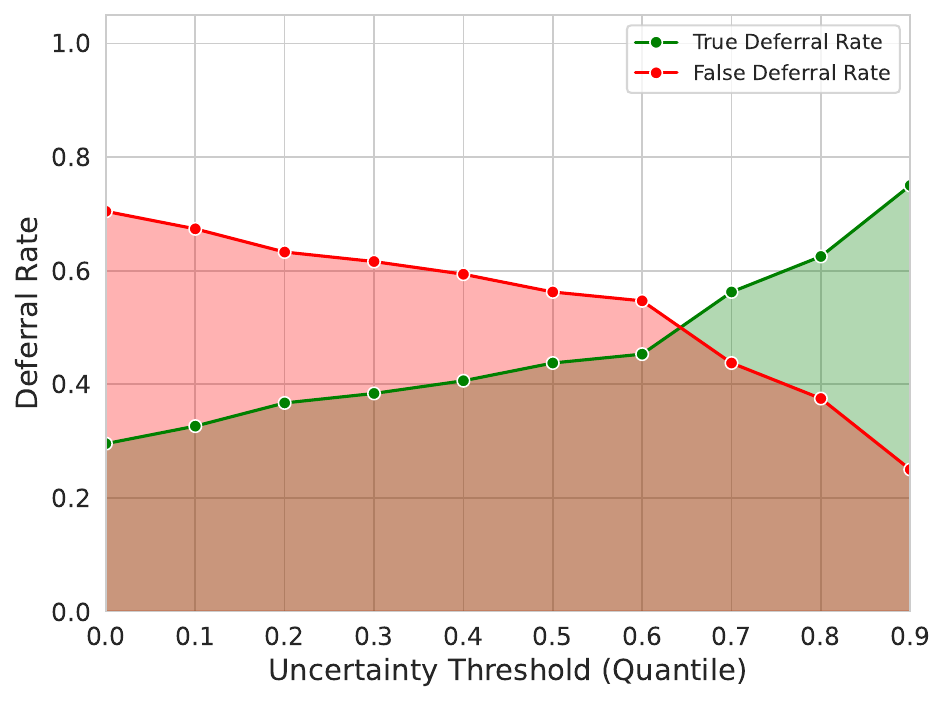}
        \caption{TDR/FDR-Uncertainty}
        \end{subfigure}
    \caption{Decision Making with uncertainty when Video is missing}
    \label{fig:decision5}
\end{figure}

\begin{figure}[th]
    \centering
    \begin{subfigure}[b]{0.45\textwidth}
        \centering
        \includegraphics[width=\textwidth]{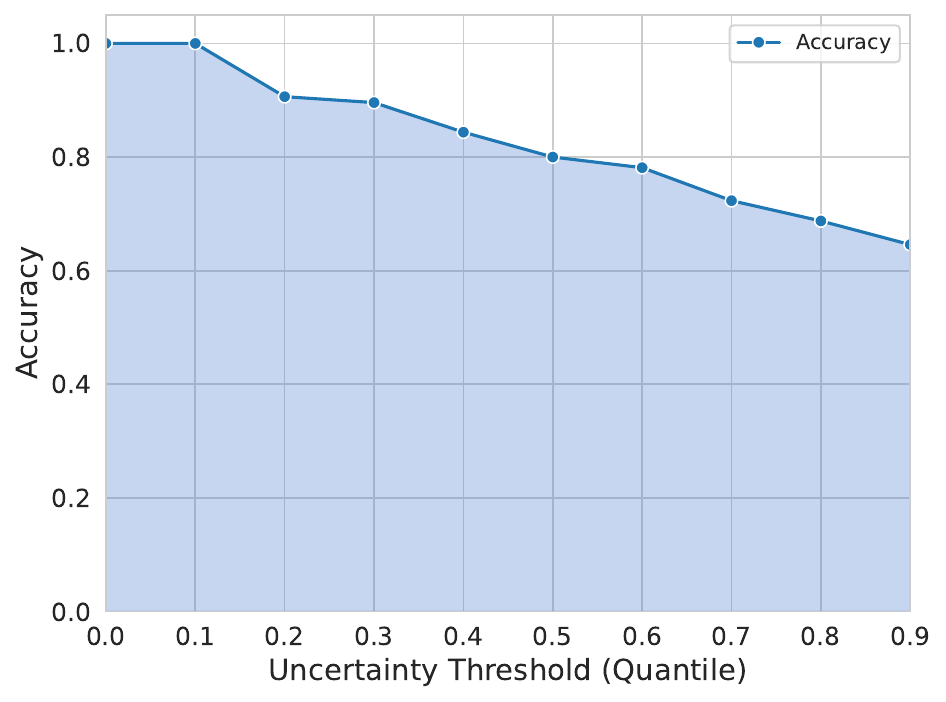}
        \caption{Accuracy-Uncertainty}
        \end{subfigure}
        \hspace{1em}
    \begin{subfigure}[b]{0.45\textwidth}
        \centering
        \includegraphics[width=\textwidth]{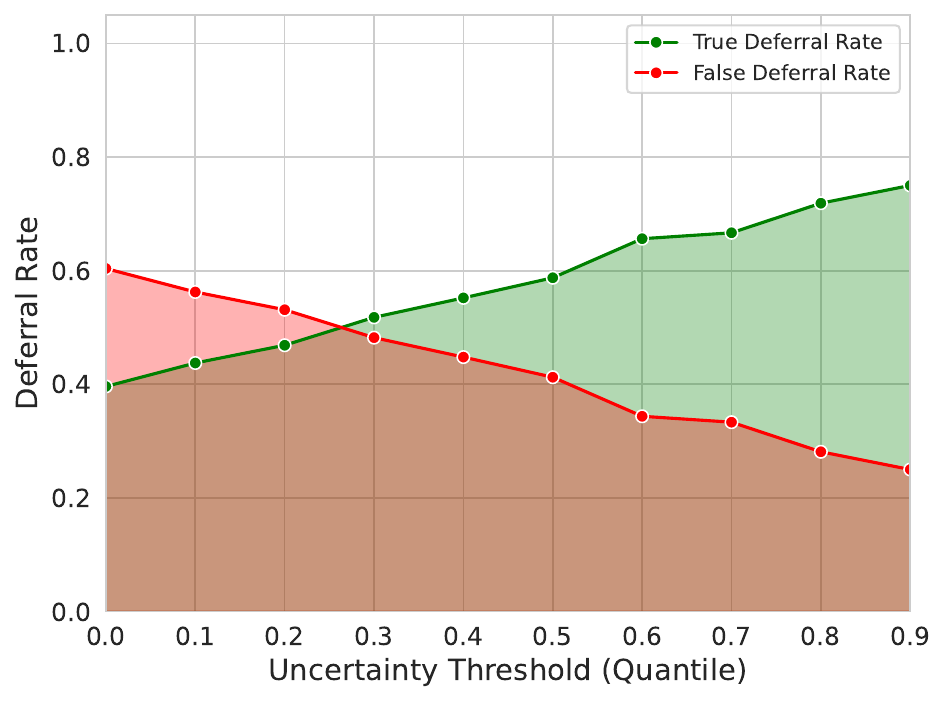}
        \caption{TDR/FDR-Uncertainty}
        \end{subfigure}
    \caption{Decision Making with uncertainty when Accel is missing}
    \label{fig:decision3}
\end{figure}

\begin{figure}[th]
    \centering
    \begin{subfigure}[b]{0.45\textwidth}
        \centering
        \includegraphics[width=\textwidth]{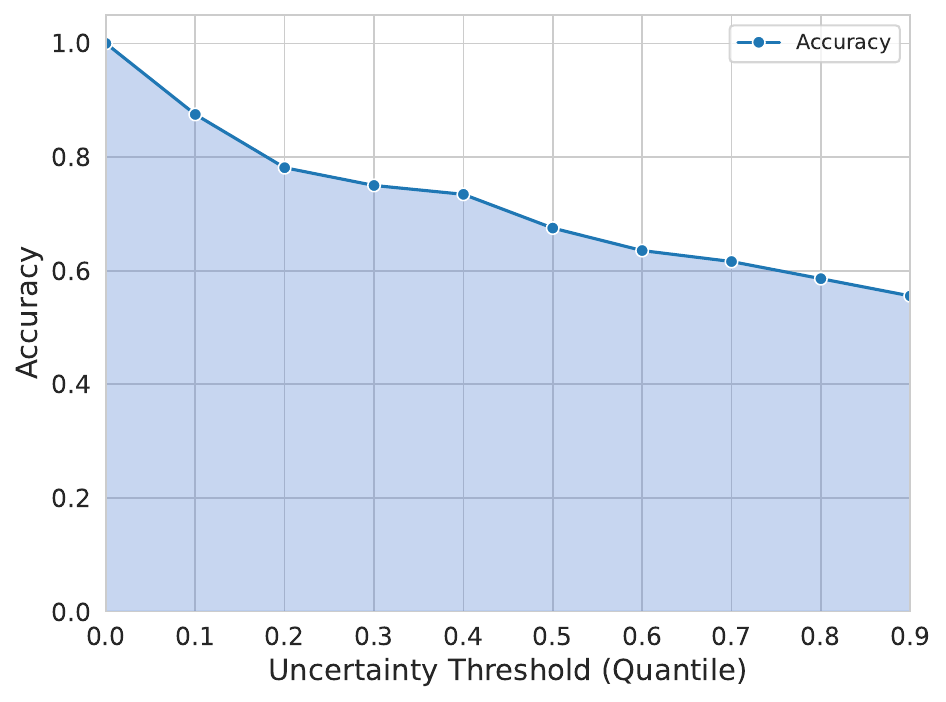}
        \caption{Accuracy-Uncertainty}
        \end{subfigure}
        \hspace{1em}
    \begin{subfigure}[b]{0.45\textwidth}
        \centering
        \includegraphics[width=\textwidth]{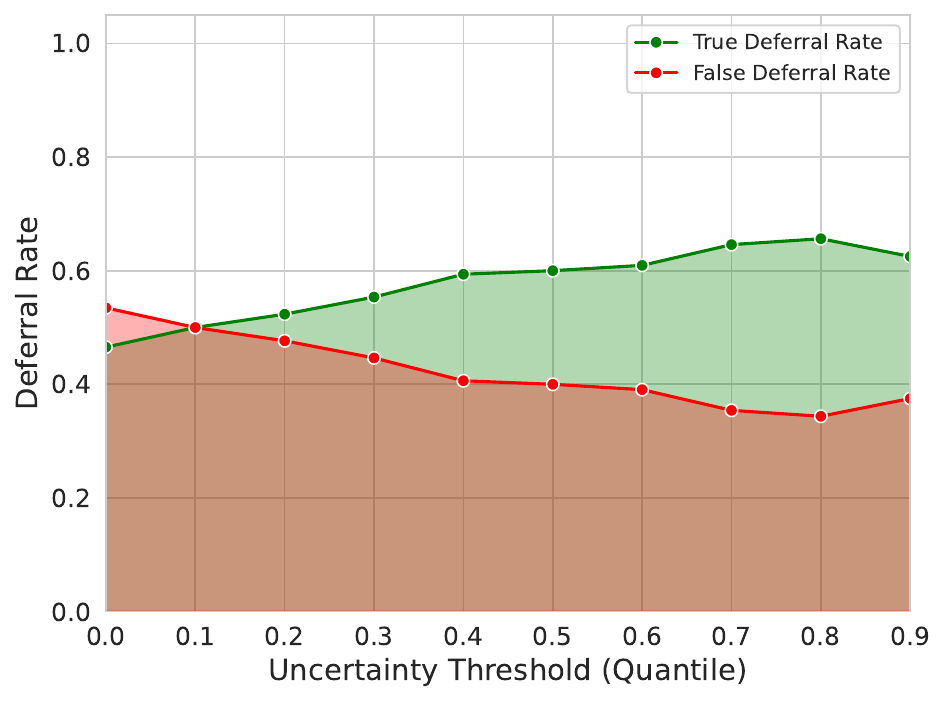}
        \caption{TDR/FDR-Uncertainty}
        \end{subfigure}
    \caption{Decision Making with uncertainty when Gyro is missing}
    \label{fig:decision4}
\end{figure}

Building on the main text analysis, we simulate the decision-making process on the UTD-MHAD dataset under conditions where different modalities are missing (Figures \ref{fig:decision5}, \ref{fig:decision3}, \ref{fig:decision4}). Each figure represents the inference scenarios when the Video, Accel, or Gyro modality is absent. Similar to the decision-making process with full modalities, incorporating uncertainty estimates in cases with missing modalities continues to guide a reliable decision-making process by adjusting different uncertainty thresholds.

\begin{figure}
    \centering
    \includegraphics[width=0.8\linewidth]{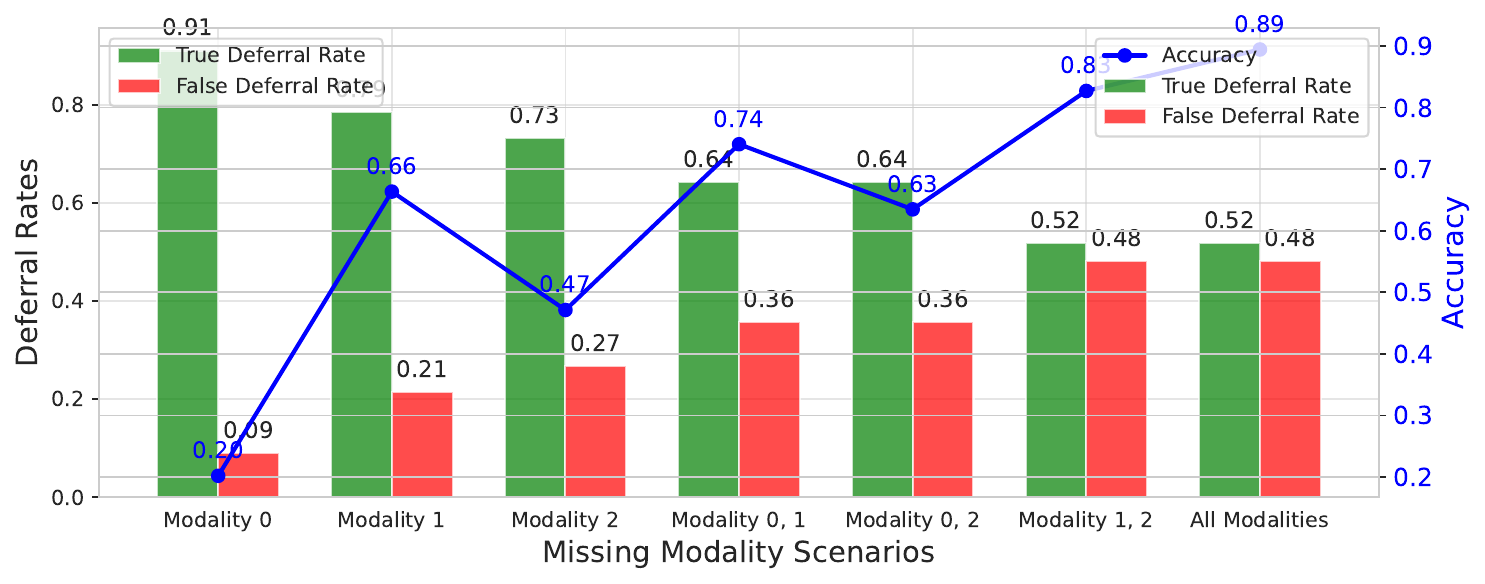}
    \caption{Decision making analysis with different input modalities combinations on UTD-MHAD dataset, with defer threshold set to be 0.65 quantile.}
    \label{fig:decision_making_aggregate}
\end{figure}
Lastly, for an aggregated perspective, the figure \ref{fig:decision_making_aggregate} shows the combined decision-making analysis on the UTD-MHAD dataset across different input modality combinations, with a deferral threshold set at the $0.65$ quantile of output uncertainty. The green bars represent the true deferral rate, which indicates the proportion of incorrect predictions successfully deferred, while the red bars indicate the false deferral rate, representing the proportion of correct predictions unnecessarily deferred. The blue line shows the accuracy of the non-deferred predictions. From this visualization, the model effectively defers incorrect predictions when modality $0$ is missing (true deferral rate $\sim 0.91$), but performance declines as more modalities are removed, with a notable drop for modality $2$ (Gyro). False deferral rates remain low but vary slightly across scenarios at the chosen threshold ($0.65$ uncertainty quantile), suggesting that the optimal uncertainty threshold may differ depending on the input modality combination. Non-deferred prediction accuracy decreases significantly when critical modalities like modality $2$ are missing, underscoring its importance for robust performance. While the deferral strategy effectively reduces errors, further optimization of uncertainty thresholds is required to adapt to varying input modalities.

\subsubsection{Additional Comparison with Prompt-based techniques}
\label{sec:app_prompt_tuning}
\input{tables/mosi_additional}

We further compare our SURE pipeline with two representative approaches that use prompt-based tuning techniques to address missing modalities \cite{prompt_1, prompt_2}. Similar to our work, these approaches also leverage pretrained multimodal pipelines for efficient training. Their key innovation lies in introducing trainable prompts to indicate the presence of missing modalities.

\textbf{Setting.} The chosen task for demonstration is Semantic Analysis task. In line with the CMU-MOSI experiment described in the main text, both frameworks are implemented using the MMML \cite{mmml} model, pretrained on the CMU-MOSEI dataset \cite{cmu_mosei}. To ensure a fair comparison, all core modules from the original codebases of the two approaches are preserved to accurately replicate their performance. The training dataset is designed similarly to the main experiment, with $50\%$ of modalities randomly masked and treated as missing.

\textbf{Result.} As shown in Table \ref{tab:mosi_extend}, SURE outperforms the two prompt-based approaches in handling missing modalities, achieving better downstream task performance. This advantage may stem from the limited number of learnable parameters introduced by these techniques, which likely constrain their ability to adapt effectively to scenarios with missing modalities.

%% file: sections_edited/literature.tex
\subsection{Related Works} \label{sec:related_works}

\textbf{Multimodal missing modalities.} Recent research has focused on developing models resilient to missing modalities \cite{aaai_21, cvpr22, icml22, actionmae, prompt_1, missing_add_1, missing_add_2, missing_add_3}. Key directions include: (1) Contrastive loss-based methods that align latent spaces for cross-modal knowledge transfer \cite{cvpr22,missing_1,kdd20}; (2) Generative approaches, such as VAE-based models \cite{vae_1} or latent space reconstruction \cite{actionmae}, to approximate missing inputs; and (3) Prompt-based techniques \cite{prompt_1, prompt_2, prompt_3} that use trainable prompts to adapt models to various combinations of missing modalities.
In the first direction, representative methods like Smil \cite{aaai_21} employ Bayesian meta-learning to approximate latent features for incomplete data, while GMC \cite{icml22} ensures geometric alignment in multimodal representations, allowing unimodal data to substitute for missing modalities.
For the second direction, ActionMAE, inspired by the masked autoencoder framework \cite{nips22, eccv_22}, reconstructs latent representations of missing modalities by randomly dropping feature tokens and learning to predict them.
In the third direction, Lee et al. \cite{prompt_1} propose missing-aware prompts that are integrated into pretrained multimodal transformers during training, enabling models to handle missing modalities effectively during evaluation.
While these approaches show promise in specific scenarios, they often rely heavily on large-scale datasets and lack robust mechanisms for quantifying uncertainty in the presence of missing modalities.  SURE addresses these gaps by leveraging pretrained models that require less data and providing a robust framework for estimating uncertainty in both reconstructed inputs and downstream predictions, enhancing reliability and interpretability.

\textbf{Uncertainty Estimation.}
Recent methods for uncertainty estimation in predictions primarily rely on Bayesian models \cite{deepensemble, kendall2017uncertainties}. However, while these models can estimate uncertainty, their predictive performance often lags behind other approaches. Some post-hoc works have explored using Laplace approximation to estimate uncertainty \cite{daxberger2021laplace, eschenhagen2021mixtures}, but these methods require computing the Hessian matrix, making them infeasible for high-dimensional problems \cite{fu2018deep}. Another direction involves test-time data augmentation \cite{wang2019aleatoric, ayhan2018test}, where multiple outputs are perturbed to estimate uncertainty. However, this approach is sometimes poorly calibrated, which is critical for accurate uncertainty estimation \cite{gawlikowski2023survey}.
SURE offers a more efficient alternative by estimating uncertainty without compromising predictive performance on downstream tasks. Unlike Laplace approximation, SURE avoids computational issues in high-dimensional spaces, and it does not rely on test-time perturbations, ensuring better-calibrated uncertainty estimates across diverse settings. Additionally, SURE imposes no assumptions on the output size, making it more flexible for a variety of applications.

%% file: algorithms/training.tex
\begin{algorithm}[ht]
  \caption{SURE training process} \label{alg:training}
  \textbf{Input}: \\
    $\triangleright \hspace{0.2cm} \mathcal{D}_{train} = \{(\mathbf{x}^i_k); \mathbf{y}_k | i \in \mathcal{I}_k - \text{set of indices for available modalities in sample } k^{th}\}$. \\
    $\triangleright \hspace{0.2cm} f^i(.)$ - frozen pretrained projectors; $r^i(.)$ - reconstruction modules $(i=1, \dots, M)$.\\
    $\triangleright \hspace{0.2cm} \omega(.)$ - frozen pretrained fusion module; $g(.)$ - classifier head. \\
    \textbf{Output}: \\
    $\triangleright \hspace{0.2cm} r^{i*}(.)$ - Trained reconstruction modules; $g^{*}(.)$ - Trained classifier head $(i=1, \dots, M)$. \\
  \begin{algorithmic}[1]
    \begin{minipage}[t]{0.4\linewidth}%
    \raggedright
      \myState{\textit{Initialize} $r^{i}(.)$'s and $g(.)$} \\
      \Comment{Train reconstruction modules} 
      \For{mini-batch $\mathcal{B} \in \mathcal{D}_{train}$} \\
        \myState{$l_{rec} \leftarrow 0;$} \\
        \For{$i \in \{1,\dots,M\}$} \\
            \myState{$l_{rec}^i \leftarrow 0;$} \\
            \For{$j \in \{1,\dots,M\}; j \neq i$} \\
                \myState{$\mathbf{z}^i_k = f^i(\mathbf{x}^i_k) \quad (\forall k: i \in \mathcal{I}_k)$;}
                \myState{$\Tilde{\mathbf{z}}^i_k, \Tilde{\sigma}^i_k \leftarrow r^i(\mathbf{z}^j_k) \quad (\forall k: i, j \in \mathcal{I}_k)$;}
                \myState{$l_{rec}^i \leftarrow l_{rec}^i + \mathcal{L}_{rec}(\mathbf{z}_i; \mathbf{z}_j)$;} \\
            \EndFor
            \myState{$l_{rec} \leftarrow l_{rec} + l_{rec}^i;$}
        \EndFor
        \myState{Backprop with $l_{rec}$;} 
        \myState{Optimizer step;}
  \EndFor
  \end{minipage}\hfill%
  \begin{minipage}[t]{0.48\linewidth}%
  \raggedright
    \myState{Freeze reconstructed modules $r^i(.)$;} \\
    \Comment{Train classifier head} 
      \For{mini-batch $\mathcal{B} \in \mathcal{D}_{train}$} \\
        \myState{$\mathbf{z}^i_k = f^i(\mathbf{x}^i_k) \quad (\forall i: i \in \mathcal{I}_k)$;}
        \myState{For $\forall i,j; i \notin \mathcal{I}_k, j \in \mathcal{I}_k$:}
         \myState{$\quad \Tilde{\mathbf{z}}^i_{j-k}, \Tilde{\sigma}^i_{j-k} = r^i(\mathbf{x}^j_k)$);}
        \myState{$\quad \Tilde{\mathbf{z}}^i_k = \texttt{average}(\Tilde{\mathbf{z}}^i_{j-k})$);}
        \myState{$\quad \Tilde{\sigma}_{\Tilde{z}_i}^2 = \texttt{average}(\Tilde{\sigma}^i_{j-k})$;}
        \myState{$\quad \Tilde{\mathbf{y}}_k, \Tilde{\sigma}_{\omega-k} \leftarrow g(\omega(\mathbf{z}^i_k, \Tilde{\mathbf{z}}^j_k))$}
        \myState{$\Tilde{\sigma}_{input-k} \leftarrow \sum_{i \notin \mathcal{I}_k}\left(\frac{\partial \omega}{\partial \Tilde{z}^i_k}\right)^2 \Tilde{\sigma}_{\Tilde{z}_i}^2$;}
        \myState{$\Tilde{\sigma}_{\Tilde{y}_k}^2 \leftarrow \Tilde{\sigma}_{input-k} + \Tilde{\sigma}_{\omega-k}$;}
        \myState{$l_{downstream} \leftarrow \mathcal{L}_{downstrean}(\Tilde{\mathbf{y}}_k; \mathbf{y}_k)$;}
        \myState{$l_{y-pcc} \leftarrow \mathcal{L}_{PCC}(\Tilde{\sigma}_{\Tilde{y}_k}^2; l_{downstream})$;}
        \myState{Backprop with $l_{y-pcc}$ and $l_{downstream}$; Optimizer step;}
  \EndFor  
  \end{minipage}
  \end{algorithmic}
\end{algorithm}

%% file: tables/utd_mhad_full.tex
\begin{table}[th]
\centering
\caption{Results of different approaches on UTD-MHAD Dataset given every possible combination of input modalities.}
\label{tab:utd_mhad_full}
\resizebox{0.8\textwidth}{!}{%
\begin{tabular}{@{}lllcccc@{}}
\toprule
                                          & \textbf{Model}                                                                   &               & \textbf{F1}                           & \textbf{Acc}                          & \textbf{\begin{tabular}[c]{@{}c@{}}Reconstruct \\ Uncertainty Corr\end{tabular}} & \textbf{\begin{tabular}[c]{@{}c@{}}Output \\ Uncertainty Corr\end{tabular}} \\ \midrule
                                          &                                                                                  & Video         & 0.044                                 & 0.059                                 & -                                                                                & -                                                                           \\
                                          &                                                                                  & Accel         & 0.204                                 & 0.231                                 & -                                                                                & -                                                                           \\
                                          &                                                                                  & Gyro          & 0.303                                 & 0.311                                 & -                                                                                & -                                                                           \\
                                          &                                                                                  & Video + Accel & 0.034                                 & 0.085                                 & -                                                                                & -                                                                           \\
                                          &                                                                                  & Video + Gyro  & 0.301                                 & 0.305                                 & -                                                                                & -                                                                           \\
                                          &                                                                                  & Accel + Gyro  & 0.31                                  & 0.306                                 & -                                                                                & -                                                                           \\
                                          & \multirow{-7}{*}{ActionMAE}                                                      & Full          & 0.531                                 & 0.537                                 & -                                                                                & -                                                                           \\ \cmidrule(l){2-7} 
                                          &                                                                                  & Video         & 0.069                                 & 0.033                                 & -                                                                                & -                                                                           \\
                                          &                                                                                  & Watch Accel   & 0.473                                 & 0.408                                 & -                                                                                & -                                                                           \\
                                          &                                                                                  & Phone Gyro    & 0.52                                  & 0.472                                 & -                                                                                & -                                                                           \\
                                          &                                                                                  & Video + Accel & 0.524                                 & 0.449                                 & -                                                                                & -                                                                           \\
                                          &                                                                                  & Video + Gyro  & 0.536                                 & 0.553                                 & -                                                                                & -                                                                           \\
                                          &                                                                                  & Accel + Gyro  & 0.577                                 & 0.586                                 & -                                                                                & -                                                                           \\
                                          & \multirow{-7}{*}{DiCMoR}                                                         & Full          & 0.653                                 & 0.636                                 & -                                                                                & -                                                                           \\ \cmidrule(l){2-7} 
                                          &                                                                                  & Video         & 0.089                                 & 0.069                                 & -                                                                                & -                                                                           \\
                                          &                                                                                  & Watch Accel   & 0.157                                 & 0.158                                 & -                                                                                & -                                                                           \\
                                          &                                                                                  & Phone Gyro    & 0.141                                 & 0.145                                 & -                                                                                & -                                                                           \\
                                          &                                                                                  & Video + Accel & 0.152                                 & 0.152                                 & -                                                                                & -                                                                           \\
                                          &                                                                                  & Video + Gyro  & 0.248                                 & 0.257                                 & -                                                                                & -                                                                           \\
                                          &                                                                                  & Accel + Gyro  & 0.316                                 & 0.278                                 & -                                                                                & -                                                                           \\
\multirow{-21}{*}{\rotatebox[origin=c]{90}{Modal Reconstruction}}   & \multirow{-7}{*}{IMDer}                                                          & Full          & 0.687                                 & 0.689                                 & -                                                                                & -                                                                           \\ \midrule
                                          &                                                                                  & Video         & 0.116                                 & 0.074                                 & 0.166                                                                            & 0.122                                                                       \\
                                          &                                                                                  & Watch Accel   & 0.433                                 & 0.381                                 & 0.115                                                                            & 0.476                                                                       \\
                                          &                                                                                  & Phone Gyro    & 0.468                                 & 0.387                                 & 0.056                                                                            & 0.147                                                                       \\
                                          &                                                                                  & Video + Accel & 0.432                                 & 0.443                                 & 0.104                                                                            & 0.237                                                                       \\
                                          &                                                                                  & Video + Gyro  & 0.462                                 & 0.502                                 & 0.095                                                                            & 0.143                                                                       \\
                                          &                                                                                  & Accel + Gyro  & 0.639                                 & 0.67                                  & 0.242                                                                            & 0.29                                                                        \\
                                          & \multirow{-7}{*}{\begin{tabular}[c]{@{}l@{}}SURE + \\ Gaussian MLE\end{tabular}} & Full          & 0.693                                 & 0.651                                 & -                                                                                & 0.292                                                                       \\ \cmidrule(l){2-7} 
                                          &                                                                                  & Video         & 0.156                                 & 0.09                                  & 0.122                                                                            & 0.136                                                                       \\
                                          &                                                                                  & Watch Accel   & 0.473                                 & 0.404                                 & 0.135                                                                            & 0.486                                                                       \\
                                          &                                                                                  & Phone Gyro    & 0.595                                 & 0.571                                 & 0.171                                                                            & 0.223                                                                       \\
                                          &                                                                                  & Video + Accel & 0.452                                 & 0.52                                  & 0.186                                                                            & 0.292                                                                       \\
                                          &                                                                                  & Video + Gyro  & 0.546                                 & 0.56                                  & 0.101                                                                            & 0.376                                                                       \\
                                          &                                                                                  & Accel + Gyro  & 0.618                                 & 0.639                                 & 0.201                                                                            & 0.417                                                                       \\
                                          & \multirow{-7}{*}{\begin{tabular}[c]{@{}l@{}}SURE + \\ MC DropOut\end{tabular}}   & Full          & 0.739                                 & 0.718                                 & -                                                                                & 0.512                                                                       \\ \cmidrule(l){2-7} 
                                          &                                                                                  & Video         & 0.25                                  & 0.207                                 & 0.249                                                                            & 0.126                                                                       \\
                                          &                                                                                  & Watch Accel   & 0.468                                 & 0.453                                 & 0.175                                                                            & 0.421                                                                       \\
                                          &                                                                                  & Phone Gyro    & 0.593                                 & 0.604                                 & 0.122                                                                            & 0.436                                                                       \\
                                          &                                                                                  & Video + Accel & 0.652                                 & 0.662                                 & 0.092                                                                            & 0.346                                                                       \\
                                          &                                                                                  & Video + Gyro  & {\color[HTML]{0000FF} 0.776}          & {\color[HTML]{0000FF} 0.781}          & 0.176                                                                            & 0.462                                                                       \\
                                          &                                                                                  & Accel + Gyro  & {\color[HTML]{FF0000} \textbf{0.839}} & {\color[HTML]{FF0000} \textbf{0.843}} & 0.278                                                                            & 0.486                                                                       \\
\multirow{-21}{*}{\rotatebox[origin=c]{90}{Uncertainty Estimation}} & \multirow{-7}{*}{\begin{tabular}[c]{@{}l@{}}SURE + \\ DeepEnsemble\end{tabular}} & Full          & 0.737                                 & 0.735                                 & -                                                                                & 0.481                                                                       \\ \midrule
                                          &                                                                                  & Video         & 0.161                                 & 0.121                                 & {\color[HTML]{FF0000} \textbf{0.878}}                                            & 0.226                                                                       \\
                                          &                                                                                  & Watch Accel   & 0.462                                 & 0.431                                 & 0.837                                                                            & {\color[HTML]{0000FF} 0.53}                                                 \\
                                          &                                                                                  & Phone Gyro    & 0.607                                 & 0.59                                  & 0.863                                                                            & 0.306                                                                       \\
                                          &                                                                                  & Video + Accel & 0.542                                 & 0.606                                 & {\color[HTML]{0000FF} 0.873}                                                     & 0.412                                                                       \\
                                          &                                                                                  & Video + Gyro  & 0.609                                 & 0.637                                 & 0.862                                                                            & 0.379                                                                       \\
                                          &                                                                                  & Accel + Gyro  & 0.679                                 & 0.706                                 & 0.455                                                                            & 0.51                                                                        \\
                                          & \multirow{-7}{*}{\textbf{SURE}}                                                  & Full          & 0.739                                 & 0.74                                  & -                                                                                & {\color[HTML]{FF0000} \textbf{0.568}}                                       \\ \bottomrule
\end{tabular}
}
\end{table}

%% file: tables/mosi_additional.tex
\begin{table}[th]
\centering
\caption{Additional results of different approaches on CMU-MOSI Dataset.}
\label{tab:mosi_extend}
\resizebox{0.9\textwidth}{!}{%
\begin{tabular}{@{}lcccccccccccc@{}}
\toprule
\textbf{Model} & \multicolumn{3}{c}{\textbf{MAE}}                                                                                      & \multicolumn{3}{c}{\textbf{Corr}}                                                                                     & \multicolumn{3}{c}{\textbf{F1}}                                                                                       & \multicolumn{3}{c}{\textbf{Acc}}                                                                                     \\ \midrule
               & T(ext)                                & A(udio)                               & F(ull)                                & T                                     & A                                     & F                                     & T                                     & A                                     & F                                     & T                                     & A                                     & F                                    \\
MPMM           & 0.683                                 & 1.197                                 & 0.668                                 & 0.83                                  & 0.495                                 & 0.834                                 & {\color[HTML]{0000FF} 0.87}           & {\color[HTML]{0000FF} 0.69}           & 0.874                                 & {\color[HTML]{0000FF} 0.871}          & {\color[HTML]{0000FF} 0.689}          & 0.875                                \\
MPLMM          & {\color[HTML]{0000FF} 0.624}          & {\color[HTML]{0000FF} 1.166}          & {\color[HTML]{0000FF} 0.607}          & {\color[HTML]{0000FF} 0.838}          & {\color[HTML]{0000FF} 0.509}          & {\color[HTML]{0000FF} 0.842}          & 0.865                                 & {\color[HTML]{FF0000} \textbf{0.697}} & {\color[HTML]{0000FF} 0.879}          & 0.865                                 & {\color[HTML]{FF0000} \textbf{0.694}} & {\color[HTML]{0000FF} 0.879}         \\
\textbf{SURE}           & {\color[HTML]{FF0000} \textbf{0.602}} & {\color[HTML]{FF0000} \textbf{1.148}} & {\color[HTML]{FF0000} \textbf{0.583}} & {\color[HTML]{FF0000} \textbf{0.865}} & {\color[HTML]{FF0000} \textbf{0.557}} & {\color[HTML]{FF0000} \textbf{0.869}} & {\color[HTML]{FF0000} \textbf{0.896}} & 0.685                                 & {\color[HTML]{FF0000} \textbf{0.891}} & {\color[HTML]{FF0000} \textbf{0.894}} & 0.684                                 & {\color[HTML]{FF0000} \textbf{0.89}} \\ \bottomrule
\end{tabular}
}
\end{table}